\documentclass[twoside]{article}
\usepackage[accepted]{aistats2024}
\usepackage[utf8]{inputenc}
\usepackage{bm,bbm,amsmath,amsfonts,amssymb}
\usepackage{authblk}

\usepackage{amsthm}

\usepackage{xspace}
\usepackage{url,ifthen}
\usepackage{caption}

\usepackage{mathtools}
\usepackage{tabularx}
\usepackage{graphicx}
\usepackage{algorithmicx,algorithm}
\usepackage[algo2e]{algorithm2e}
\usepackage[font=footnotesize]{subcaption}

\usepackage{hyperref}
\usepackage{mathrsfs}
\usepackage{xcolor}
\usepackage{tikz}
\usepackage[capitalise]{cleveref}
\usepackage{thm-restate}

\newtheorem{theorem}{Theorem}
\newtheorem{lemma}[theorem]{Lemma}
\theoremstyle{definition}
\newtheorem{definition}[theorem]{Definition}

\newcommand{\eps}{\epsilon}

\newcommand{\E}{\operatorname{E}}

\newcommand{\R}{\mathcal{R}}

\DeclareMathOperator*{\argmax}{arg\,max}

\newcommand{\EE}[2]{\mathbb{E}_{#1}\left[#2\right]}

\def\E#1{\EE{\,}{#1}}
\def\R{\mathbb{R}}
\newcommand{\cN}{\mathcal{N}}

\newcommand{\norm}[1]{\| #1 \|}

\usepackage[round]{natbib}

\begin{document}

\twocolumn[

\aistatstitle{On the Privacy of Selection Mechanisms with Gaussian Noise}

\aistatsauthor{ Jonathan Lebensold \And Doina Precup \And  Borja Balle }

\aistatsaddress{ McGill University, Mila \And McGill University, Mila, Google DeepMind \And Google DeepMind } ]

\begin{abstract}
  Report Noisy Max and Above Threshold are two classical differentially private (DP) selection mechanisms. Their output is obtained by adding noise to a sequence of low-sensitivity queries and reporting the identity of the query whose (noisy) answer satisfies a certain condition. Pure DP guarantees for these mechanisms are easy to obtain when Laplace noise is added to the queries. On the other hand, when instantiated using Gaussian noise, standard analyses only yield approximate DP guarantees despite the fact that the outputs of these mechanisms lie in a discrete space. In this work, we revisit the analysis of Report Noisy Max and Above Threshold with Gaussian noise and show that, under the additional assumption that the underlying queries are bounded,  it is possible to provide pure ex-ante DP bounds for Report Noisy Max and pure ex-post DP bounds for Above Threshold. The resulting bounds are tight and depend on closed-form expressions that can be numerically evaluated using standard methods.  Empirically we find these lead to tighter privacy accounting in the high privacy, low data regime. Further, we propose a simple privacy filter for composing pure ex-post DP guarantees, and use it to derive a fully adaptive Gaussian Sparse Vector Technique mechanism. Finally, we provide experiments on mobility and energy consumption datasets demonstrating that our Sparse Vector Technique is practically competitive with previous approaches and requires less hyper-parameter tuning.
\end{abstract}

\section{INTRODUCTION}

Differential Privacy (DP) \citep{Dwork2006} has become the standard framework used for the private release of sensitive statistics. In particular, DP has been embraced by industry and governments to guarantee that potentially sensitive statistics cannot be linked back to individual users. For example, during the COVID-19 pandemic, Google Maps mobility data was published with DP \citep{Aktay2020-fi} in order for public health authorities to better understand various curve-flattening measures (such as work-from-home, or shelter-in-place). 
Recently, the Wikimedia Foundation deployed DP for their page-level visit statistics \citep{Desfontaines_tumult_wiki-nz}.

Underpinning the design of differentially private mechanisms is a fundamental trade-off between privacy and utility characterized by the Fundamental Law of Information Recovery stating that ``overly accurate answers to too many questions will destroy privacy in a spectacular way'' \citep{Dwork2014}. In practice this imposes a limit on how many statistics can be privately released to a desired accuracy within a pre-specified privacy budget. In applications where the space of possible statistics is too large to allow for a full private and accurate release, analysts can overcome the privacy-utility trade-off by identifying and releasing only those statistics that contain relevant information.
This is necessary for example in periodic data collection where users contribute many times to a dataset and relevant statistics must be released repeatedly (e.g.\ to report temporal trends, change points, extreme events, etc.) \citep{hu2021human, Wang2016-lf, Xu2017-on}.
A notable example is the use of smart meters in energy grids, where statistics can help manage electrical demand and encourage smoother consumption but can also be used to infer information like income, occupancy, etc. \citep{acs2011have, Bohli2010-bk, Haji_Mirzaee2022-mf}.

Private selection mechanisms aim to identify relevant statistics. This problem is usually framed as query selection: an analyst defines a collection of queries against the target dataset, and a private mechanism is used to identify queries returning ``abnormally'' large values. Two notable settings arise: the offline case, where all the queries can be specified in advance, and the online case, where the analyst can adaptively select queries based on the previous ones.
In the offline setting, one of the best known private selection mechanisms is \emph{Report Noisy Max} whereby the analyst submits a collection of low-sensitivity scalar queries. The mechanism then adds noise to the values of all the queries and returns the index of the query attaining the largest (noisy) value. The Exponential Mechanism \citep{Dwork2014} and Permute-and-flip are commonly used for offline selection and are generally considered best-in-class \citep{mckenna2020permute}. 

In the online setting, the cornerstone selection mechanism is \emph{Above Threshold}, where the analyst submits a threshold and a sequence of (potentially adaptive) low-sensitivity scalar queries to which the mechanism iteratively computes answers, until a noisy value exceeds the target threshold.
Running Above Threshold repeatedly is called the Sparse Vector Technique \citep{Dwork2009-od}, and is a common privacy primitive in change-point detection, empirical risk minimization, density estimation and other online learning algorithms \citep{Zhang2021-by, Ligett2017-accfirst, Dwork2014, cummings2018differential}. 

When Laplace noise is used, Above Threshold and Report Noisy Max offer \emph{pure} privacy guarantees -- the strongest type of DP guarantee which does not have to account for some small failure probability. The use of the Laplace distribution also has the advantage of making the mathematical analysis of the privacy guarantees relatively simple; however, since the noise distribution is less concentrated than the Gaussian distribution, it can lead to a less accurate mechanism.
Replacing the Laplace distribution with a Gaussian distribution is advantageous in many DP mechanisms, including online query selection \citep{Abadi2016-bm, zhu2020improving, papernot2018scalable}.

\paragraph{Contributions.}
In this paper we revisit the privacy analysis of the Above Threshold mechanism instantiated with Gaussian noise. Our main observation is that under the mild assumption that the queries submitted are uniformly bounded (in addition to low-sensitivity), we can provide pure DP guarantees that can be computed using standard numerical tools. In particular, we provide pure \emph{ex-post}\footnote{This is a guarantee that depends on the output produced by the mechanism; see \Cref{sec:svt} for details.} DP guarantees for Gaussian Above Threshold. In the process, we also develop an \emph{ex-ante} DP guarantee for Gaussian Report Noisy Max. Our analysis relies on identifying the worst case values for the query answers on a pair of neighboring datasets, a technique which might be of independent interest. Empirically, we find that these privacy bounds lead to tighter privacy accounting in the high privacy, low data regime, when compared to other Gaussian-based mechanisms.

Further, we define a meta-algorithm that composes Gaussian Above Threshold with ex-post DP guarantees. Thus, we derive a fully-adaptive Sparse Vector Technique (SVT), which we call \emph{Filtered Self-Reporting Composition} (FSRC). Our method is particularly appealing since an analyst need not choose the number of releases or the maximum number of queries up front.

Finally, we provide experiments on mobility and energy consumption datasets demonstrating that our analyses yield mechanisms that in practice match or outperform previous approaches. 

\section{PRELIMINARIES}

\paragraph{Differential Privacy.} Throughout we will be considering randomized mechanisms operating on some dataset $D \in \mathcal{X}$. 
Two datasets $D$ and $D'$ are said to be \textit{neighboring}, denoted $D \simeq D'$, if they differ in the data of a single individual (e.g. one user is added, removed, or replaced by another).

\begin{definition}[DP \citep{Dwork2006Calibrating}]
\label{def:approxdp}
Let $\epsilon \geq 0$ and $\delta \geq 0$. A randomized mechanism $M : \mathcal{X} \rightarrow \mathcal{O}$, is $(\epsilon, \delta)$-DP if for every pair of neighboring datasets $D \simeq D'$ and every subset $S \subseteq \mathcal{O}$, we have:
\begin{align*}
  \label{eq:puredp}
  \Pr [ M(D) \in S ] \leq e^{\epsilon} \Pr [ M(D') \in S ] + \delta \enspace.
\end{align*}
When $\delta = 0$, we write $\epsilon$-DP and say the mechanism satisfies \emph{pure DP}; otherwise we say that it satisfies \emph{approximate DP}. If the output space $\mathcal{O}$ is discrete, then a mechanism is $\epsilon$-DP if and only if $\Pr [ M(D) = o ] \leq e^{\epsilon} \Pr [ M(D') = o ]$ for all $o \in \mathcal{O}$ and $D \simeq D'$.
\end{definition}

A key feature of DP is its resilience to \emph{post-processing}: any function of the output of a DP mechanism still satisfies DP with the same (or better) parameters. A common way to establish that a mechanism satisfies differential privacy is through a high-probability bound on the privacy loss random variable.

\begin{restatable}[Privacy Loss \citep{Dinur2003}]{definition}{privacyloss}
    \label{def:privacy_loss}
    Let $M : \mathcal{X} \to \cal{O}$ be a randomized mechanism and consider neighboring datasets $D \simeq D'$. The \emph{privacy loss} of the pair $D$ and $D'$ under $M$ for a given output $o \in \cal{O}$ is defined as
    \begin{align*}
        \mathcal{L}_{M, D, D'} (o) &= 
        \log 
            \frac{
            \Pr[M(D) = o]
            }{
            \Pr[M(D') = o]
        }
        \enspace.
    \end{align*}
\end{restatable}

\begin{definition}[pDP\citep{Kasiviswanathan2008-ug}]
\label{lem:pDP_dp}
A mechanism $M: \mathcal{X} \to \cal{O}$ satisfies $(\epsilon, \delta)$-pDP if for any $D \simeq D' \in \cal{X}$,
\begin{align*}
    \Pr_{o \sim M(D)} \left [ \mathcal{L}_{M, D, D'} (o) > \epsilon \right ] \leq \delta \enspace.
\end{align*}
\end{definition}

If a mechanism is $(\epsilon, \delta)$-pDP, then it is also $(\epsilon, \delta)$-DP \citep{Kasiviswanathan2008-ug}.
A simple way to obtain pDP guarantees is through bounds on the moment-generating function of the privacy loss random variable; this is one of the motivations for Rényi Differential Privacy (RDP), which relies on a bound of the Rényi divergence between two distributions.

\begin{definition}[Rényi divergence~\citep{Renyi1961-st}]
  Let $\alpha > 1$. The Rényi divergence of order $\alpha$
  between two probability distributions $P$ and $Q$ on $\mathcal{X}$ is defined by:
\begin{equation*}\label{eq:Renyi_divergences}
  \mathbb{D}_{\alpha}(P||Q) \triangleq \frac{1}{\alpha - 1}
  \log \mathbb{E}_{o \sim Q }
  \left [ \frac{P(o)}{Q(o)} \right ]^\alpha.
\end{equation*}
\end{definition}

\begin{definition}[Rényi DP~\citep{mironov2017renyi}]
\label{def:rdp}
Let $\alpha > 1$ and $\epsilon \geq 0$. A randomized mechanism $M$ is $(\alpha, \epsilon)$-RDP for all $D \simeq D'$ if, $\mathbb{D}_{\alpha}(M(D) \| M(D') ) \leq \epsilon $.
\end{definition}

It is possible to convert RDP guarantees into probabilistic and approximate DP guarantees \citep{mironov2017renyi,DBLP:conf/aistats/BalleBGHS20}.
In particular, any $(\alpha, \epsilon)$-RDP mechanism satisfies $(\epsilon_p, \delta)$-pDP guarantees with $\epsilon_p = \epsilon + \log(1/\delta)/(\alpha-1)$ by Markov's inequality. RDP greatly simplifies the analysis of mechanisms based on Gaussian noise because the R{\'e}nyi divergence between Gaussian distributions has a simple expression, as well as  a simple analysis under composition.

\paragraph{Gaussian Mechanism.} Adding Gaussian noise to the result of a low-sensitivity query is a staple of differentially private mechanism design.
Let $q : \mathcal{X} \to \R^d$ be a query with global sensitivity $\Delta_q = \sup_{D \simeq D'} \norm{q(D) - q(D')}_2$ and $Z \sim \mathcal{N}(0, \sigma^2 I)$. The Gaussian Mechanism defined as $M(D) = q(D) + Z$
satisfies $(\alpha, \epsilon)$-RDP with $\epsilon = \frac{\alpha \Delta_q^{2}}{2 \sigma^2}$ for every $\alpha > 1$ \citep{mironov2017renyi}.
It is possible to convert this RDP guarantee into an approximate DP guarantee, although tighter approximate DP bounds can be obtained directly \citep[Theorem 8]{balle2018_agm}.
\paragraph{Gaussian Report Noisy Max.}
In some applications it is useful to privately select which among a collection of queries $q_1, \ldots q_d : \mathcal{X} \to \R$ (approximately) attains the largest value on a given dataset. This leads to the \emph{report noisy max} mechanism -- when instantiated using Gaussian noise, the mechanism is given by,
$M(D) = \argmax_{i \in [d]} q_i(D) + Z_i$,
where $Z_1, \ldots, Z_d \sim \mathcal{N}(0, \sigma^2)$.
Since the $\argmax$ operation is merely a post-processing of the Gaussian mechanism applied to the $d$-dimensional query $q = (q_1, \ldots, q_d) : \mathcal{X} \to \R^d$ with sensitivity $\Delta_q$, the Gaussian Report Noisy Max mechanism inherits the same privacy guarantees as the Gaussian mechanism above (e.g.\ the bounds provided by \citet[Theorem 8]{balle2018_agm} or the RDP to DP conversion). Note that if each of the individual queries has sensitivity bounded by $\Delta$, then we have $\Delta_q \leq \sqrt{d} \Delta$.

\paragraph{Gaussian Above Threshold.} The Above Threshold mechanism receives a sequence of low-sensitivity queries and privately identifies the first query (approximately) exceeding a given threshold. The mechanism was introduced by \citet{Dwork2009-od} and forms the basis of the Sparse Vector Technique, a composition of Above Threshold algorithms to find a sparse set of relevant queries among a large set. The standard version of the Above Threshold algorithm uses Laplace noise, and its privacy analysis is notoriously subtle \citep{Lyu2016-ox}. \citet{zhu2020improving} recently proposed an RDP analysis which can be applied to the Gaussian version of Above Threshold (see \Cref{alg:Gaussian_at}).

\begin{algorithm}[H]
    \SetAlgoLined
    \SetKwInOut{Input}{input}
    \Input{dataset $D$; noise parameters $\sigma_X, \sigma_Z$; a stream of queries $q_1, q_2, \ldots$; threshold $\rho$.}
    
    $\hat{\rho} = \rho + \mathcal{N}(0, \sigma_X^2)$
    
    \For{$t = 1, 2, \ldots$}{
        
        $\hat{q_t} = q_t(D) + \mathcal{N}(0, \sigma_Z^2)$
        
        \uIf{$\hat{q_t} \geq \hat{\rho}$}{

            Output $x_t = \top$ and HALT
        }
        \Else{
        
            Output $x_t = \bot$
            
        }
    }
    \caption{Gaussian Above Threshold \citep{zhu2020improving}}
    \label{alg:Gaussian_at}
\end{algorithm}

\begin{theorem}[General RDP Bound on Gaussian Above Threshold \citep{zhu2020improving}]
    \label{thm:zw2020_thm8} 
    Suppose all queries given to the mechanism $M$ in Algorithm~\ref{alg:Gaussian_at} have sensitivity bounded by $\Delta$. Then for $\gamma > 1$, $\infty > \alpha >1$, 
    \begin{align*}
        \mathbb{D}_\alpha\left(M(D) \| M(D^{\prime})\right) &\leq 
        \left ( \frac{\gamma}{\gamma-1} \right )
        \left ( \frac{\alpha \Delta^2}{2 \sigma_X^2} \right ) \\
        &+\frac{2 \alpha \Delta^2}{\sigma_Z^2} 
        +\frac{\log  \mathbb{E}_{x \sim \mathcal{N}(0, \sigma_X^2)}\left[\mathbb{E}[ T \mid x]^\gamma\right] }{\gamma(\alpha-1)} ,
    \end{align*}
     where $T$ is a random variable indicating the stopping time of $M(D)$.
\end{theorem}

Unlike in the Laplace-based Above Threshold, the privacy bound for the Gaussian case depends on how long it takes the mechanism to stop (the third term in the expression above). When the queries are known a priori to be non-negative, it is possible to obtain bounds on the running time to provide the RDP guarantee below.

\begin{theorem}[Gaussian Above Threshold RDP \citep{zhu2020improving}]
    \label{lem_Gaussian_at}
    
    Gaussian Above Threshold, with $\sigma_Z \geq \sqrt{3} \sigma_X$, threshold $\rho \geq 0$, and $\Delta$-sensitive, non-negative queries, satisfies $(\alpha, \epsilon)$-RDP with
    \begin{align*}
        \epsilon = \frac{\alpha \triangle^2}{\sigma_X^2}+\frac{2 \alpha \triangle^2}{\sigma_Z^2}+\frac{\log \left(1+2 \sqrt{3} \pi\left(1+\frac{9 \rho^2}{\sigma_X^2}\right) e^{\frac{\rho^2}{\sigma_X^2}}\right)}{2(\alpha-1)} \enspace .
    \end{align*}
\end{theorem}

Alternative methods to bound the running time often include the analyst supplying a bound $k$ on the running time to the algorithm, forcing it to stop after a certain number of steps even if the threshold has not been exceeded.
\citet{zhu2020improving} also propose a number of composition results for the Sparse Vector Technique. One version allows the analyst to continue releasing queries without having to re-sample the threshold noise $\sigma_X$, but in each case the analyst must know ahead of time the number of times they wish to release a ``$\top$", and have an upper bound on the maximum number of queries they intend to run. 

\section{PURE PRIVATE GAUSSIAN MECHANISMS}\label{sec:svt}
\vspace{-0.5\baselineskip}
To introduce our main contribution, we begin with a warm-up. We propose a pure DP bound for Gaussian Report Noisy Max. The proof techniques are the same for Above Threshold, but the analysis is simpler since all the queries are symmetric.
\subsection{Warm-Up: Gaussian Report Noisy Max}\label{sec:rnmax}

The following is a pure DP bound for Gaussian Report Noisy Max for queries with bounded range.

\begin{restatable}[Pure DP for Gaussian Report Noisy Max]{theorem}{dpGNMAX}
\label{thm_expost_gnmax}

Let $M$ be the Gaussian Report Noisy Max mechanism with standard deviation $\sigma$ applied to $d > 1$ bounded queries $q_1, \ldots, q_d: \mathcal{X} \to [a, b]$, each with sensitivity bounded by $\Delta$. Let $c = b - a$. Let $\Phi(\cdot)$ be the standard Gaussian CDF. Then $M$ satisfies $\epsilon$-DP with
\begin{align*}
\epsilon
&=
\frac{
    \mathbb{E}_{z \sim \mathcal{N}(0, 1)} \left [
        \Phi \left(z - \frac{c - 2 \Delta}{\sigma}\right)^{d-1}
        \right ]
            }
            {
    \mathbb{E}_{z \sim \mathcal{N}(0, 1)} \left [
        \Phi \left (
        z - \frac{
        c
        }{
        \sigma
        }
        \right )^{d-1}
        \right ]
            }
\enspace.
\end{align*}
\end{restatable}
Notably, with the standard analysis of Gaussian Report Noisy Max, the privacy guarantee only depends on the ratio $\Delta / \sigma$. However in our case, the bound on privacy additionally depends on the range of the queries through $c/\sigma$. In fact, it is easy to see that if $c \to \infty$ then Gaussian Report Noisy Max cannot admit a pure DP bound. 

The proof, deferred to the supplemental, relies on identifying the worst-case values that uniformly bounded queries can attain on a pair of neighboring datasets. In particular, we show that the worst case is obtained on a pair of neighboring datasets $D$ and $D'$ such that $q_1(D) = \cdots = q_{d-1}(D) = b - \Delta$, $q_d(D) = a + \Delta$, $q_1(D') = \cdots = q_{d-1}(D') = b$, and $q_d(D') = a$.

Note that the classical approach sketched in the previous section applied to our setting would consider the privacy of a Gaussian mechanism with a $d$-dimensional query of sensitivity $\Delta_q = \Delta \sqrt{d}$, and treat the $\argmax$ operation as a post-processing step.
While our approach gives a pure DP guarantee, the classical approach does not because Gaussian noise by itself cannot provide that guarantee. However, the classical approach gives a bound that is easy to compute. 

The RDP approach gives a simple closed form expression, while the direct approximate DP approach gives a bound that has a simple dependence on the CDF of a standard Gaussian \citep{balle2018_agm}. In contrast, the bound provided by our result requires estimating Gaussian expectations of a complex function; unfortunately, this does not admit a simple closed form expression. In ~\cref{sec:numerics} we evaluate numerical methods for computing this quantity and compare the resulting values of $\epsilon$ with those provided by the classical approach.
\vspace{-1.5\baselineskip}
\begin{center}
\begin{figure*}
    \centering
    \includegraphics[width=1\linewidth]{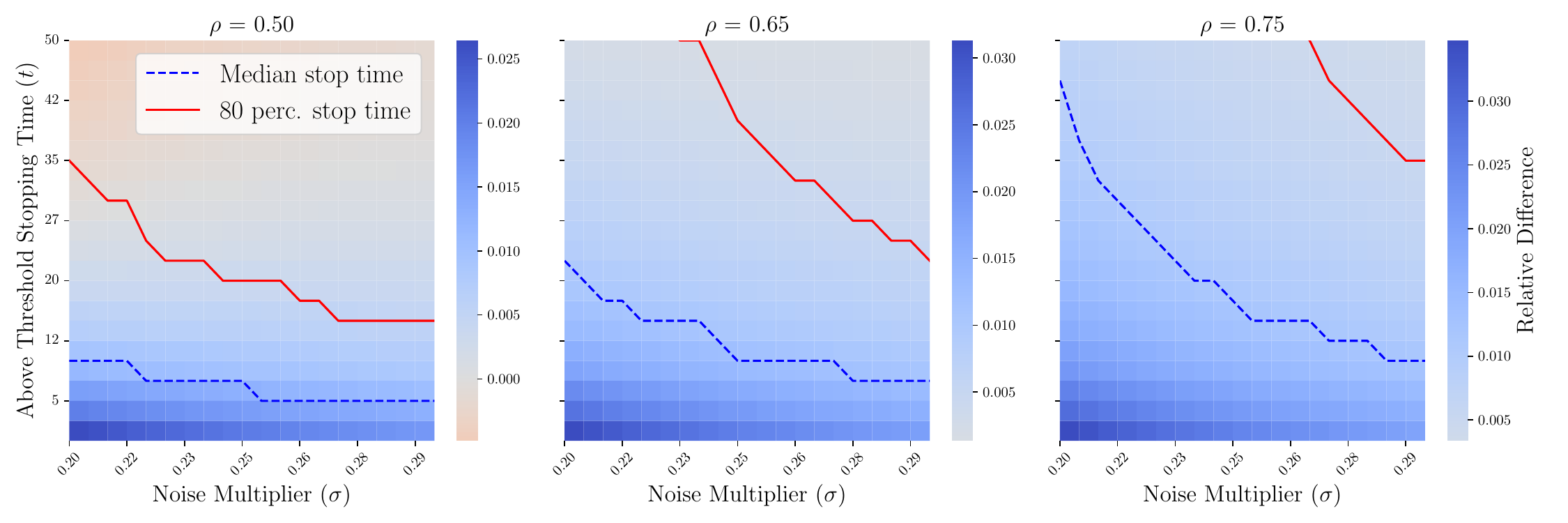}
    \caption{
    Privacy accounting for $\Delta = 1\mathrm{e}{-3}, \delta = 1\mathrm{e}{-5}$. The heatmap shows where the ex-post Above Threshold Analysis offers an improvement over the Gaussian Above Threshold. As a comparison, we simulate expected stopping times for a range of multipliers, for $[a,b] = [0, 1]$. The \textcolor{blue}{blue dotted line} corresponds to the median stopping time when simulating 10k trials with a worst-case dataset. The \textcolor{red}{red line} corresponds to the 80th percentile.
    The plot shows a range of hyper-parameters as well as where the worst-case dataset is likely to halt. Our bounds provide improvements over the baseline below the blue line when squares are blue.
    }    
    \label{fig:at-heatmap}
\end{figure*}
\end{center}
\vspace{-0.4\baselineskip}

\subsection{Ex-Post Gaussian Above Threshold}
In the preceding section, we introduced a method to judge the privacy of Above Threshold before execution (\cref{lem_Gaussian_at}) -- these are often referred to as \emph{ex-ante} privacy guarantees. Alternatively, we can measure the privacy loss after execution, leading to so-called \emph{ex-post} privacy guarantees. 

\begin{definition}[Ex-Post DP \citep{Ligett2017-accfirst}]
Consider a function $\epsilon_{\text{p}}: \cal{O} \to \R_+ \cup \{\infty\}$.
A randomized mechanism $M : \mathcal{X} \to \cal{O}$ satisfies $\epsilon_{\text{p}}$-ex-post-DP if for any possible output $o \in \cal{O}$ and any pair of neighboring datasets we have $\mathcal{L}_{M, D, D'}(o) \leq \epsilon_{\text{p}}(o)$.
\end{definition}

Since Above Threshold outputs the (approximate) halting time, we can exploit the difference between when the ex-post and ex-ante privacy guarantee. In \cref{fig:at-heatmap}, we observe when the ex-post analysis is most beneficial, and when it is most likely to occur. As the public threshold increases, so does the separation between the ex-ante analysis and the ex-post analysis. With Above Threshold, we also need to consider how the mechanism might perform against a worst-case dataset maximizing the stopping time. Since queries are non-negative and bounded, this occurs when $q_1(D) = \cdots = q_t(D) = 0$. We then compute the median stopping time (dashed blue line) and the 80th percentile (solid red line) for the worst case dataset. For bounded range between zero and one, we see the greatest effect when the mechanism halts early and $\rho$ is calibrated to be closer to one.

When the queries have bounded range, Gaussian Above Threshold (\cref{alg:Gaussian_at}) admits the following ex-post privacy bound.
\begin{restatable}[Pure Ex-post Gaussian Above Threshold]{theorem}{expostAT}
\label{thm_expost_at_f}
Let
\begin{align*}
    \psi_{\xi}(x) &= \Phi \left ( \left(x + \rho - (b + a) + \xi \right ) / \sigma_Z \right ) \enspace,
    \text{and, } \\
    \beta_{\xi}(x) &= \Phi\left ( \left ( x - \rho + a + \xi \right ) / \sigma_Z \right ) \enspace.
\end{align*}
Given a stream $q_1, q_2, \ldots : \mathcal{X} \to [a,b]$ with global sensitivity $\Delta$, the Gaussian Above Threshold mechanism (Algorithm~\ref{alg:Gaussian_at}), halting at time step $t$ with $o=\{\bot^{t-1} \top \}$, satisfies $\epsilon_{\text{p}}$-ex-post-DP with
\begin{align*}
    \epsilon_{\text{p}}(o)  
        &= 
    \frac{
\mathbb{E}_{x \sim \mathcal{N}(0, 1)} \left [
    \psi_{\Delta}(\sigma_X x)^{(t - 1)}
    \cdot 
    \beta_{\Delta}(\sigma_X x)
    \right ]
        }
        {
\mathbb{E}_{x \sim \mathcal{N}(0, 1)} \left [
    \psi_{0}(\sigma_X x)^{(t - 1)}
    \cdot 
    \beta_{0}(\sigma_X x)
    \right ]
        }
\enspace.
\end{align*}
\end{restatable}

Note that this formulation is very similar to the ratio of expectations in our pure DP analysis of Report Noisy Max (\cref{thm_expost_gnmax}). The proof, deferred to the supplemental, follows from Pure DP Gaussian Report Noisy Max, where we find that the point where the ratio is maximized is the same in all but the final time step.  In the last step, the query bound is negated. Hence, the ratio is largest when $q_t = a$. A final remark is that the mechanism's greatest privacy loss occurs when each prior step would have been more optimal than the step at which it halts.

\subsection{Fully-Adaptive Composition with Ex-Post Privacy Guarantees}\label{sec:fsrc}

Above Threshold stops the first time the threshold is exceeded. The Sparse Vector Technique (SVT) applies a sequence of Above Threshold mechanisms to find a set of queries that (approximately) exceed the pre-defined threshold. In the case of Laplace noise, the privacy analysis of SVT can be performed either directly (if the noise applied to the threshold is not refreshed after each Above Threshold terminates) or via composition \citep{Dwork2014,Lyu2016-ox}. Something similar holds for SVT with Gaussian noise, although in this case the analysis without noise resampling only applies to a range of parameters \citep{zhu2020improving}. Here we present a simple technique for fully adaptive composition of mechanisms that have both pDP and ex-post-DP guarantees, which can be combined with our analysis of Gaussian Above Threshold to yield a fully adaptive SVT with Gaussian noise algorithm without additional hyper-parameters like the maximum number of queries per Above Threshold or the maximum number of invocations of the Above Threshold mechanism. 

Our method (\cref{alg:fsrc_meta}) sequentially composes a stream of adaptive mechanisms and applies a stopping time rule that limits over-spending a predefined privacy budget. The privacy expenditure of each mechanism is tracked based on their output, using the ex-post privacy guarantee. The stopping rule uses the probabilistic DP guarantee to halt when there is a high enough probability of exceeding the privacy budget.
\vspace{-1\baselineskip}

\begin{center}
\begin{figure*}
    \centering
    \includegraphics[width=1\linewidth]{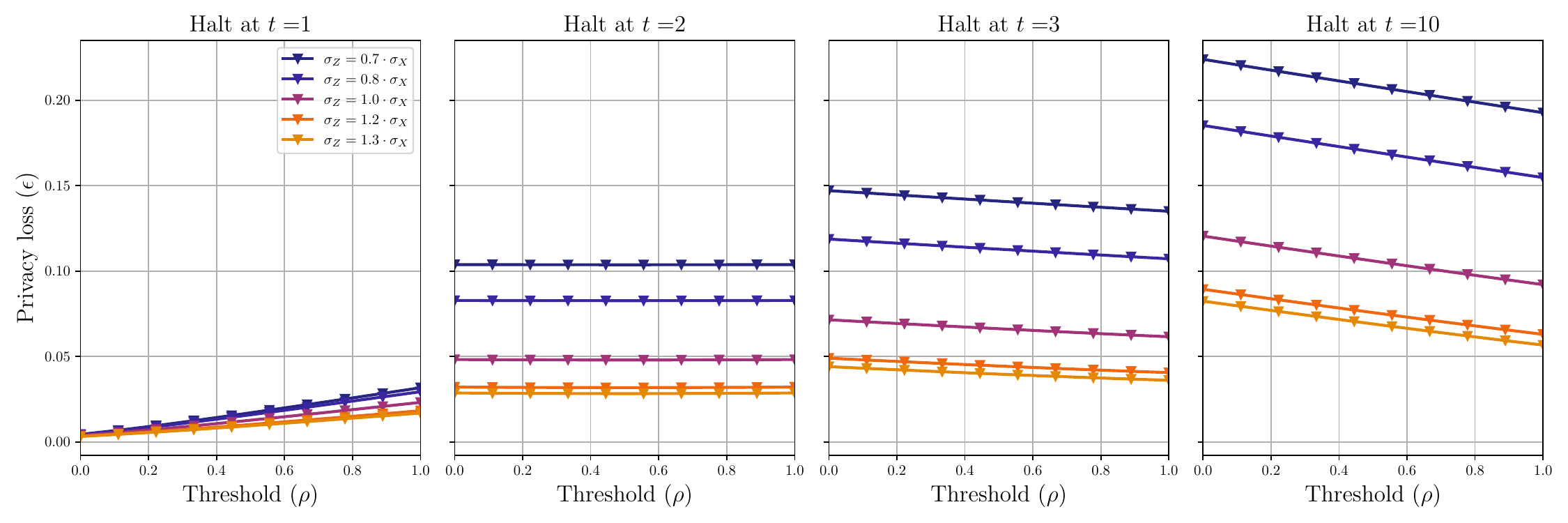}
    \caption{
    Ex-Post Above Threshold Privacy Loss ($\Delta = 0.001, \sigma_X = 0.15$). The ex-post privacy loss also changes as a function of the public threshold $\rho$. Note that if the mechanism halts after two timesteps, the minimum is observed when $\rho=0.5$. As $t$ increases, the privacy loss decreases as $\rho \to 1$.
    }
    \label{fig:at-threshold_privacy}
\end{figure*}
\end{center}
\begin{algorithm}[H]
\SetAlgoLined
\SetKwInOut{Input}{input}
\Input{dataset $D$; privacy budget $\epsilon$; a stream of adaptive mechanisms $M_1, M_2, \ldots$; a stream of values $\epsilon_{\max,1}, \epsilon_{\max, 2}, \ldots$; and stream of functions $\epsilon_{\text{p},1}, \epsilon_{\text{p},2}, \ldots$}
\For{$t = 1, 2, \ldots$}{
    \uIf{$\sum_{i=1}^{t-1} \epsilon_i + \epsilon_{\max,t} \geq \epsilon$}{
        HALT
    }
    \Else{
        $o_t = M_t(D; o_{1:t-1})$
        
        $\eps_t = \epsilon_{\text{p},t}(o_t)$
        
        Release $o_t$
    }
}
\caption{Filtered Composition, Ex-Post Privacy}
\label{alg:fsrc_meta}
\end{algorithm}
    
\begin{restatable}{theorem}{fsrcEpsDP}
\label{thm:fsrc_algorithm_dp}

Suppose the mechanisms $M_1, M_2, \ldots$ provided to 
\Cref{alg:fsrc_meta} are such that $M_t$ is $(\epsilon_{\max,t}, \delta)$-pDP and $\epsilon_{\text{p},t}$-ex-post-DP for all $t$.
Then \Cref{alg:fsrc_meta} satisfies $(\epsilon, \delta)$-DP.
\end{restatable}

In \cref{fig:at-threshold_privacy} we plot  the ex-post privacy loss bound as a function of the public threshold parameter, $\rho$. Given a stream of queries $i=1, 2, \ldots, $ bounded by $a \leq q_i \leq b$, the mechanism will report less privacy loss as $\rho \to b$. 

\subsection{Numerical Computation of  Bounds}\label{sec:numerics}

We evaluate two methods for producing numerical estimates of the bounds: Monte Carlo and numerical integration.
Monte Carlo density estimation is a natural starting point for computing bounds for \cref{thm_expost_gnmax} and \cref{thm_expost_at_f}. However, Monte Carlo methods require a tremendous amount of samples to yield accurate results. Note that in both cases, the CDF is taken to an exponential power in the numerator and the denominator, causing numerical instabilities. To improve runtime performance and stability, we use a scientific library that can compute integrals with high precision. The Python mpmath \citep{mpmath} library is able to compute each density and the outer integral with arbitrary degree of precision.

\begin{figure}[H]
    \centering
    \includegraphics[width=1\linewidth]{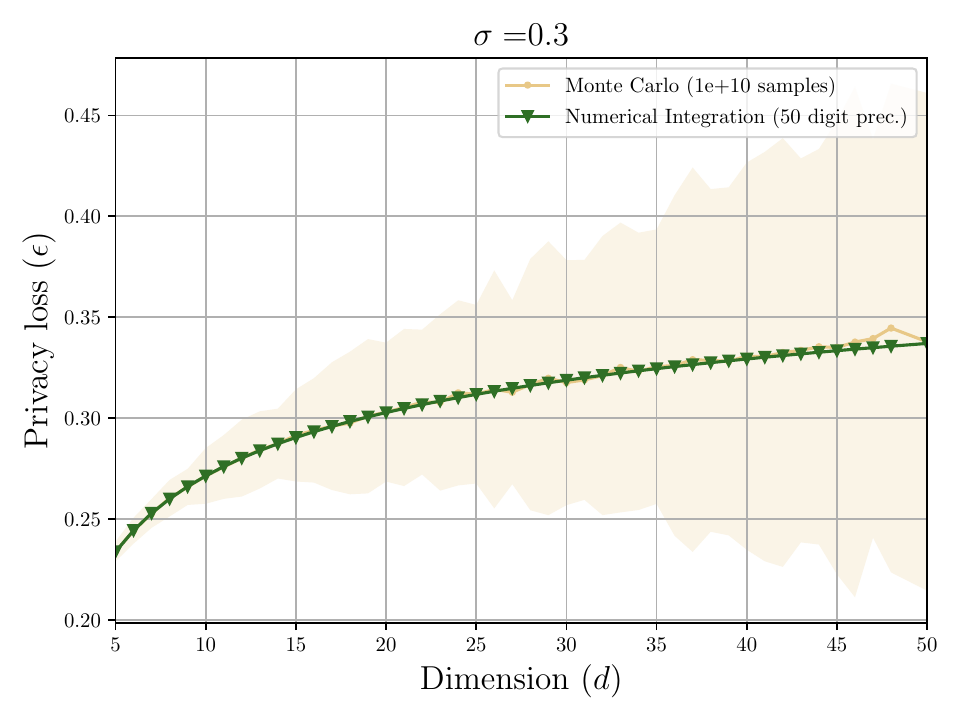}
    \caption{
    Gaussian Report Noisy Max for $\Delta=0.01$.
    Numerical integration (green) compared to Monte Carlo estimate (beige) with 10B samples. Shaded region is the standard deviation over 100 trials. Numerical integration methods are deterministic; error bars only apply to Monte Carlo estimates, which are known to converge to the true estimate with infinite samples.\looseness=-1
    }
    \label{fig:gnmax-num-estimate}
\end{figure}

As shown in \cref{fig:gnmax-num-estimate}, for low sensitivity queries $\Delta$ and $\sigma = 0.3$, the variance between trials becomes unwieldy, even when we rely on common open source libraries that can precisely estimate the CDF of a univariate Gaussian.

\begin{figure*}
    \centering
    \includegraphics[width=1\linewidth]{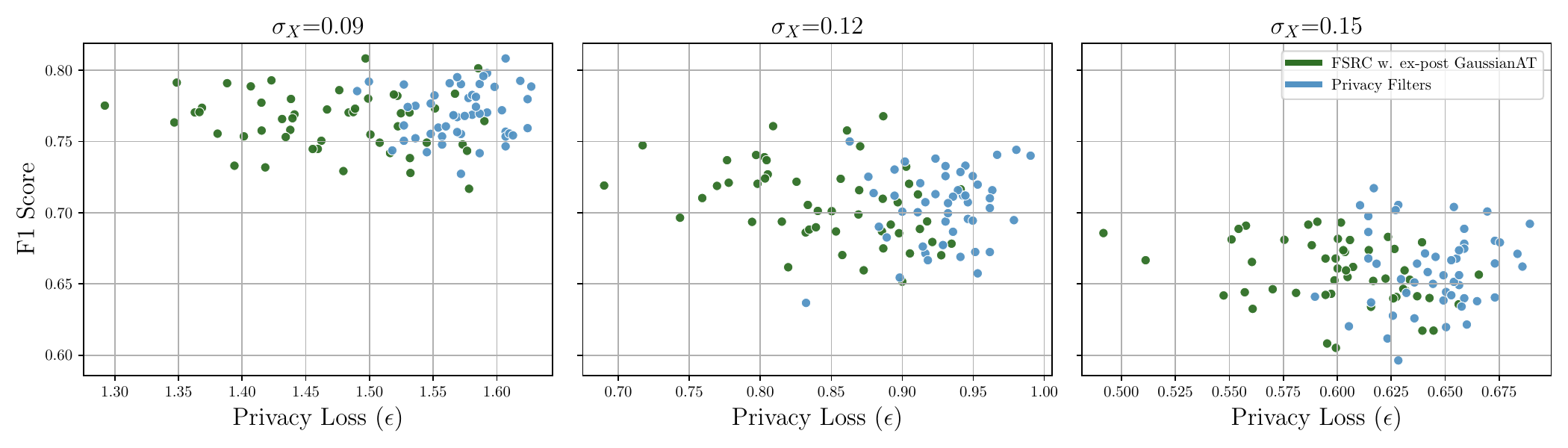}
    \caption{
    Scatter plot indicating the accuracy for UCI Bikes with $\rho = 0.575$. and final privacy spend over a range of noise multipliers. Final privacy loss ($\epsilon$) is reported for FSRC (green). Threshold noise, $\sigma_X$ in the range of $[0.09, 0.15]$. A clear separation in privacy accounting occurs over a range of noise multipliers.}
    \label{fig:fsrc_uci_scatter}
\end{figure*}

\section{EMPIRICAL RESULTS}
We benchmark our SVT-like method (FSRC and Gaussian Above Threshold) on mobility and energy datasets. Additional experiments with the Pure DP Gaussian Report Noisy Max bound are included in the appendix. Experiments were done on a Apple M1 processor (32 GB), except for the Monte Carlo numerical estimation, which was done on a NVIDIA V100 GPU with 32 GB VRAM.

\textbf{The UCI Bikes Dataset} captures the utilization of shared bikes in a geographic area over the course of a year \citep{ucibikesdataset}.
Since we do not know the upper bound on registered customers, we take the maximum (6,946 users) and assume that this is a public value. Note that our analysis is still worst-case, in that a user is assumed to be contributing to each daily (normalized) count query. We set Above Threshold to HALT when bike sharing exceeds a threshold $\rho \in [0, 1]$. In \cref{fig:fsrc_uci_scatter} we plot a range of calibrations, and in the supplemental we show the privacy spend over time. \textbf{The LCL London Energy Dataset} \citep{lcl_smartmeter_london_energy_data}, consists of energy usage for $N= 5,564$ customers over $d = 829$ days. The larger number of queries \emph{increases} the privacy cost; however this is balanced by a \emph{decrease} in individual contribution to each query due to each query having $\Delta = 1/N$. In \cref{fig:fsrc_lcl_scatter} we plot a range of calibrations. As the threshold decreases, we witness more queries released and therefore a greater privacy spend. 

\begin{figure*}
    \centering
    \includegraphics[width=1.0\linewidth]{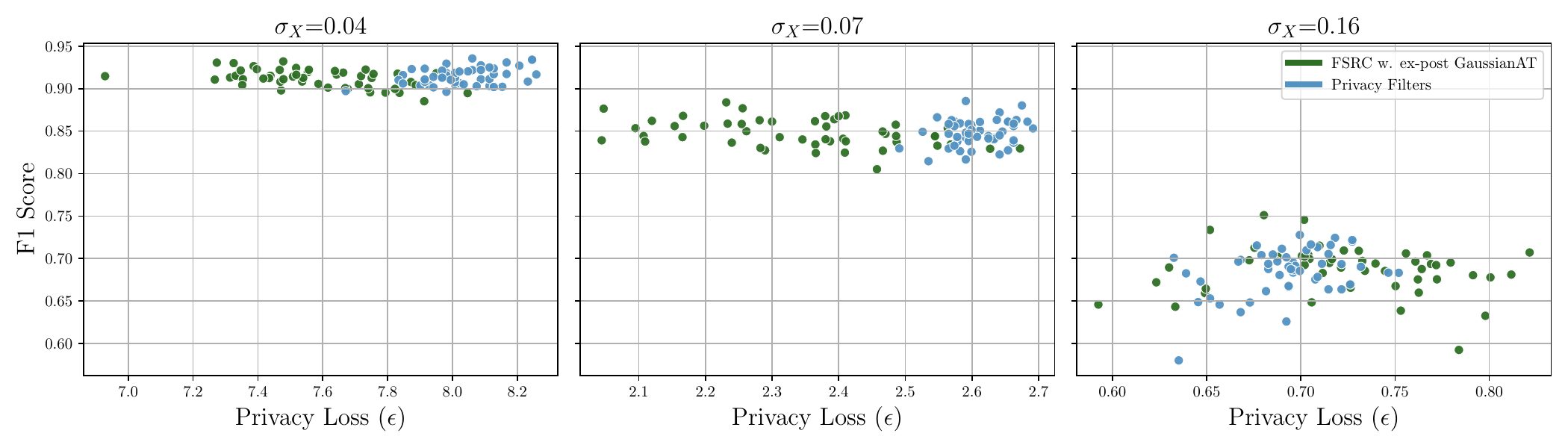}
    \caption{
    Scatter plot indicating the accuracy and final privacy spend over a range of noise multipliers for $\rho = 0.33$ with the LCL London Energy dataset. Privacy loss ($\epsilon$) is reported for FSRC (green). Threshold noise, $\sigma_X$, was evaluated in the range of $[0.04, 0.16]$. Our accounting method provides benefits when $\sigma_X = 0.04$.}
    \label{fig:fsrc_lcl_scatter}
\end{figure*}

\subsection{Online Selection with Filtered Self-Reporting Composition}
Answering sequential questions in an interactive setting typically requires the up-front selection of privacy parameters. The most common method to answer a large number of queries is the Sparse Vector Technique. In this setting, Above Threshold serves as a privacy primitive in more complex algorithms \citep{Hardt2010-la}. First, the curator decides how many times they expect to release a query in order to allocate the privacy budget. The budget is further split across two noise adding mechanisms. Noise is added to each query, as well as a public threshold parameter, which centers whether a point is flagged as a ``$\top$" (interesting) or a ``$\bot$" (not interesting). One solution---and our baseline for consideration---restricts a curator to fully-adaptive composition using a privacy filter \citep{Rogers2023-uv} and a novel result by \citet{zhu2020improving} which places no restriction on the number of queries.  We are primarily concerned with satisfying the following two requirements: (1) the analyst should be able to modify queries after Above Threshold halts, and (2) they should be able to guarantee that an up-front privacy budget is respected. Therefore, we combine FSRC with the privacy analysis in \cref{thm_expost_at_f}. To calibrate $\epsilon_{\text{max}}$, we take an RDP bound of Gaussian Above Threshold (\cref{lem_Gaussian_at}) with $\delta = 1/N$.

\paragraph{Experiment Parameters.}
In each case, the datasets have a temporal axis, meaning that when the mechanism halts, we restart at the  $t+1$' timestep  and continue accumulating privacy spend. To evaluate FSRC, we compare \cref{alg:fsrc_meta} using Gaussian Above Threshold (\cref{alg:Gaussian_at}),  to a Privacy Filter \citep{Whitehouse2022-se} with sequential composition of the same mechanism. In FSRC, we apply \cref{thm_expost_at_f} in our privacy accounting. For the baseline, we compute an RDP bound using \citet{zhu2020improving}. For each pair of noise multipliers and thresholds, we plot each result in a scatter plot, and the privacy spend over 50 runs. We use AutoDP \citep{Wang_autodp} with support for global sensitivity calibration to apply their bound with $\Delta < 1$. As in \citet{zhu2020improving}, we fix $\sigma_Z = \sqrt{3} \sigma_X$. We ran all our experiments over 50 runs with a range of values for $\sigma_X$. Importantly, the privacy spend using a DP filter cannot be disclosed without leaking privacy, and must therefore be replaced with a privacy odometer if the data curator wishes to share the spent budget \citep{rogers2016privacy}.
\textbf{Utility.}
We compute the F1 Score to measure utility for both algorithms. Since each algorithm outputs a binary vector, we can measure how many times the mechanism matches with a ground truth algorithm where no noise is added to the queries or the threshold.

\section{RELATED WORK}
Our work spans four areas of privacy research, (1) accounting methods for Gaussian mechanisms, (2) maximizing the number of interactive, user-level queries, (3) privacy filters and fully adaptive composition, and (4) ex-post privacy analysis.

\textbf{The Gaussian Mechanism} adds calibrated noise to the output of a query. In the learning setting, a number of privacy accounting techniques exist \citep{Abadi2016-bm}. However such approaches do not map directly to query selection.
In the online setting, ex-post DP accounting already exists for the Laplace Mechanism \citep{Ligett2017-accfirst}; however, many successful deployments of DP rely on Gaussian mechanisms. This is likely due to the greater noise concentration around the mean and the thinner tails exhibited by Gaussian mechanisms \citep{Dwork2014}. 

\textbf{The Sparse Vector Technique (SVT)}
  \citep{Dwork2009-od} is a foundational differentially private algorithm. By splitting the privacy budget between the cost of returning a binary vector and the query of interest, utility is increased. 
  
A Gaussian version offers better utility over the Laplace mechanism and can, surprisingly, support a potentially infinite number of queries \citep{zhu2020improving}. 
  
  \citet{Hartmann2022-bi} introduce \textit{Output DP} as a means of analyzing SVT and the Propose-Test-Release with Laplace noise. Their work generalizes results from \citet{Ligett2017-accfirst} and they show how basic composition bounds can (for few queries) offer better utility over advanced composition results. 
  
  \textbf{Ex-Post Privacy.} \citet{Ligett2017-accfirst} defined ex-post privacy in terms of $(\eps, 0)$-DP. In common with this work, they studied SVT, but with a focus on the Laplace Mechanism. \citet{redberg2021privately} consider ex-post privacy with Gaussian mechanisms for data-dependent privacy parameters with the Gaussian mechanism over real-valued queries. 
  
  \textbf{Privacy Filters}, first proposed by \citet{rogers2016privacy}, are a key ingredient in allowing adaptive composition of privacy-preserving mechanisms as well as the privacy spend. \citet{feldman2021individual}, and \citet{Lecuyer2021-vx} extended these results to Rényi DP, with the caveat that the higher order parameters needed to be pre-defined. Recently, \citet{Whitehouse2022-se} tightened these results and \citet{Rogers2023-uv} then applied ex-post privacy to probabilistic DP mechanisms. We consider their efforts to be closest to ours; however, they do not consider query online selection.

\section{CONCLUSION}
We provided pure-DP bounds on Gaussian selection mechanisms with bounded queries. Additionally, we developed new composition tools for ex-ante and ex-post privacy analysis. 
We demonstrated increased query accuracy for an equivalent privacy budget in energy and mobility datasets in the online setting. 

We consider ex-post privacy analysis a promising method tightening privacy accounting in online algorithms, and our composition result, FSRC, could be applied to other sequential algorithms.

Accounting for user-level privacy over long time horizons is necessary in a number of sequential and interactive decision-making tasks. In particular, we foresee direct benefit in applying these methods to model selection. 
\subsection*{Acknowledgements}
J.L.\ is supported by the Google DeepMind Graduate Fund and the Fonds de Recherche du Québec. We wish to thank Thomas Steinke for his comments on an early draft. We also thank Guillaume Rabusseau, Jose Gallego-Posada, Maxime Wabartha and Vincent Luczkow for many fruitful discussions. Finally, we thank Iosif Pinelis for bringing l'H\^{o}pital's Monotone Rule to our attention. 

\bibliography{main}

\begin{thebibliography}{}

\bibitem[Abadi et~al., 2016]{Abadi2016-bm}
Abadi, M., Chu, A., Goodfellow, I., McMahan, H.~B., Mironov, I., Talwar, K., and Zhang, L. (2016).
\newblock Deep learning with differential privacy.
\newblock In {\em Proceedings of the 2016 ACM SIGSAC conference on computer and communications security}, pages 308--318.

\bibitem[{\'A}cs and Castelluccia, 2011]{acs2011have}
{\'A}cs, G. and Castelluccia, C. (2011).
\newblock I have a dream!(differentially private smart metering).
\newblock In {\em International Workshop on Information Hiding}, pages 118--132. Springer.

\bibitem[Aktay et~al., 2020]{Aktay2020-fi}
Aktay, A., Bavadekar, S., Cossoul, G., Davis, J., Desfontaines, D., Fabrikant, A., Gabrilovich, E., Gadepalli, K., Gipson, B., Guevara, M., et~al. (2020).
\newblock Google covid-19 community mobility reports: anonymization process description (version 1.1).
\newblock {\em arXiv preprint arXiv:2004.04145}.

\bibitem[Anderson et~al., 1993]{Anderson1993-wh}
Anderson, G., Vamanamurthy, M., and Vuorinen, M. (1993).
\newblock Inequalities for quasiconformal mappings in space.
\newblock {\em Pacific J. Math.}, 160(1):1--18.

\bibitem[Balle et~al., 2020]{DBLP:conf/aistats/BalleBGHS20}
Balle, B., Barthe, G., Gaboardi, M., Hsu, J., and Sato, T. (2020).
\newblock Hypothesis testing interpretations and renyi differential privacy.
\newblock In Chiappa, S. and Calandra, R., editors, {\em The 23rd International Conference on Artificial Intelligence and Statistics, {AISTATS} 2020, 26-28 August 2020, Online [Palermo, Sicily, Italy]}, volume 108 of {\em Proceedings of Machine Learning Research}, pages 2496--2506. {PMLR}.

\bibitem[Balle and Wang, 2018]{balle2018_agm}
Balle, B. and Wang, Y.-X. (2018).
\newblock Improving the gaussian mechanism for differential privacy: Analytical calibration and optimal denoising.
\newblock In {\em International Conference on Machine Learning}, pages 394--403. PMLR.

\bibitem[Bohli et~al., 2010]{Bohli2010-bk}
Bohli, J.-M., Sorge, C., and Ugus, O. (2010).
\newblock A privacy model for smart metering.
\newblock In {\em 2010 {IEEE} International Conference on Communications Workshops}, pages 1--5.

\bibitem[Cummings et~al., 2018]{cummings2018differential}
Cummings, R., Krehbiel, S., Lai, K.~A., and Tantipongpipat, U. (2018).
\newblock Differential privacy for growing databases.
\newblock {\em Advances in Neural Information Processing Systems}, 31.

\bibitem[Desfontaines, 2023]{Desfontaines_tumult_wiki-nz}
Desfontaines, D. (2023).
\newblock Publishing wikipedia usage data with strong privacy guarantees.

\bibitem[Dinur and Nissim, 2003]{Dinur2003}
Dinur, I. and Nissim, K. (2003).
\newblock Revealing information while preserving privacy.
\newblock In {\em Proceedings of the Twenty-Second ACM SIGMOD-SIGACT-SIGART Symposium on Principles of Database Systems}, PODS ’03, page 202–210. Association for Computing Machinery.

\bibitem[Dwork, 2006]{Dwork2006}
Dwork, C. (2006).
\newblock {Differential privacy}.
\newblock In {\em Lecture Notes in Computer Science (including subseries Lecture Notes in Artificial Intelligence and Lecture Notes in Bioinformatics)}, volume 4052 LNCS, pages 1--12. Springer Verlag.

\bibitem[Dwork et~al., 2006]{Dwork2006Calibrating}
Dwork, C., McSherry, F., Nissim, K., and Smith, A. (2006).
\newblock {Calibrating noise to sensitivity in private data analysis}.
\newblock In {\em Lecture Notes in Computer Science (including subseries Lecture Notes in Artificial Intelligence and Lecture Notes in Bioinformatics)}, volume 3876 LNCS, pages 265--284.

\bibitem[Dwork et~al., 2009]{Dwork2009-od}
Dwork, C., Naor, M., Reingold, O., Rothblum, G.~N., and Vadhan, S. (2009).
\newblock On the complexity of differentially private data release.
\newblock In {\em Proceedings of the forty-first annual {ACM} symposium on Theory of computing}, New York, NY, USA. ACM.

\bibitem[Dwork and Roth, 2014]{Dwork2014}
Dwork, C. and Roth, A. (2014).
\newblock {\em The Algorithmic Foundations of Differential Privacy}.
\newblock Foundations and trends in theoretical computer science. Now.

\bibitem[Fanaee-T and Gama, 2013]{ucibikesdataset}
Fanaee-T, H. and Gama, J. (2013).
\newblock Event labeling combining ensemble detectors and background knowledge.
\newblock {\em Progress in Artificial Intelligence}, pages 1--15.

\bibitem[Feldman and Zrnic, 2021]{feldman2021individual}
Feldman, V. and Zrnic, T. (2021).
\newblock Individual privacy accounting via a renyi filter.
\newblock {\em Advances in Neural Information Processing Systems}, 34:28080--28091.

\bibitem[{Greater London Authority}, 2012]{lcl_smartmeter_london_energy_data}
{Greater London Authority} (2012).
\newblock {Smartmeter Energy Use Data in London Households}.

\bibitem[Haji~Mirzaee et~al., 2022]{Haji_Mirzaee2022-mf}
Haji~Mirzaee, P., Shojafar, M., Cruickshank, H., and Tafazolli, R. (2022).
\newblock Smart grid security and privacy: From conventional to machine learning issues (threats and countermeasures).
\newblock {\em IEEE Access}, 10:52922--52954.

\bibitem[Hardt and Rothblum, 2010]{Hardt2010-la}
Hardt, M. and Rothblum, G.~N. (2010).
\newblock A multiplicative weights mechanism for privacy-preserving data analysis.
\newblock In {\em 2010 {IEEE} 51st Annual Symposium on Foundations of Computer Science}. IEEE.

\bibitem[Hartmann et~al., 2022]{Hartmann2022-bi}
Hartmann, V., Bindschaedler, V., Bentkamp, A., and West, R. (2022).
\newblock Privacy accounting $\varepsilon $ conomics: Improving differential privacy composition via a posteriori bounds.
\newblock {\em arXiv preprint arXiv:2205.03470}.

\bibitem[Hu et~al., 2021]{hu2021human}
Hu, T., Wang, S., She, B., Zhang, M., Huang, X., Cui, Y., Khuri, J., Hu, Y., Fu, X., Wang, X., et~al. (2021).
\newblock Human mobility data in the covid-19 pandemic: characteristics, applications, and challenges.
\newblock {\em International Journal of Digital Earth}, 14(9):1126--1147.

\bibitem[Kasiviswanathan and Smith, 2014]{Kasiviswanathan2008-ug}
Kasiviswanathan, S.~P. and Smith, A. (2014).
\newblock On the 'semantics' of differential privacy: A bayesian formulation.
\newblock {\em Journal of Privacy and Confidentiality}, 6(1).

\bibitem[L{\'e}cuyer, 2021]{Lecuyer2021-vx}
L{\'e}cuyer, M. (2021).
\newblock Practical privacy filters and odometers with r$\backslash$'enyi differential privacy and applications to differentially private deep learning.
\newblock {\em arXiv preprint arXiv:2103.01379}.

\bibitem[Ligett et~al., 2017]{Ligett2017-accfirst}
Ligett, K., Neel, S., Roth, A., Waggoner, B., and Wu, S.~Z. (2017).
\newblock Accuracy first: Selecting a differential privacy level for accuracy constrained erm.
\newblock {\em Advances in Neural Information Processing Systems}, 30.

\bibitem[Lyu et~al., 2016]{Lyu2016-ox}
Lyu, M., Su, D., and Li, N. (2016).
\newblock Understanding the sparse vector technique for differential privacy.
\newblock {\em arXiv preprint arXiv:1603.01699}.

\bibitem[McKenna and Sheldon, 2020]{mckenna2020permute}
McKenna, R. and Sheldon, D.~R. (2020).
\newblock Permute-and-flip: A new mechanism for differentially private selection.
\newblock {\em Advances in Neural Information Processing Systems}, 33:193--203.

\bibitem[Mironov, 2017]{mironov2017renyi}
Mironov, I. (2017).
\newblock R{\'e}nyi differential privacy.
\newblock In {\em 2017 IEEE 30th computer security foundations symposium (CSF)}, pages 263--275. IEEE.

\bibitem[mpmath, 2023]{mpmath}
mpmath (2023).
\newblock {\em mpmath: a {P}ython library for arbitrary-precision floating-point arithmetic (version 1.3.0)}.
\newblock {\tt http://mpmath.org/}.

\bibitem[Narayanan and Shmatikov, 2008]{Narayanan2008-hw}
Narayanan, A. and Shmatikov, V. (2008).
\newblock Robust de-anonymization of large sparse datasets.
\newblock In {\em 2008 {IEEE} Symposium on Security and Privacy (sp 2008)}, pages 111--125.

\bibitem[Papernot et~al., 2018]{papernot2018scalable}
Papernot, N., Song, S., Mironov, I., Raghunathan, A., Talwar, K., and Erlingsson, {\'U}. (2018).
\newblock Scalable private learning with pate.
\newblock {\em arXiv preprint arXiv:1802.08908}.

\bibitem[Pinelis, 2002]{pinelis2001monotone}
Pinelis, I. (2002).
\newblock L’hospital type rules for monotonicity, with applications.
\newblock {\em J. Inequal. Pure Appl. Math}, 3(1).

\bibitem[Ratnam et~al., 2017]{neardataset}
Ratnam, E.~L., Weller, S.~R., Kellett, C.~M., and Murray, A.~T. (2017).
\newblock Residential load and rooftop pv generation: an australian distribution network dataset.
\newblock {\em International Journal of Sustainable Energy}, 36(8):787--806.

\bibitem[Redberg and Wang, 2021]{redberg2021privately}
Redberg, R. and Wang, Y.-X. (2021).
\newblock Privately publishable per-instance privacy.
\newblock {\em Advances in Neural Information Processing Systems}, 34:17335--17346.

\bibitem[R{\'e}nyi, 1961]{Renyi1961-st}
R{\'e}nyi, A. (1961).
\newblock On measures of entropy and information.
\newblock In {\em Proceedings of the Fourth Berkeley Symposium on Mathematical Statistics and Probability, Volume 1: Contributions to the Theory of Statistics}, volume~4, pages 547--562.

\bibitem[Rogers et~al., 2023]{Rogers2023-uv}
Rogers, R., Samorodnitsky, G., Wu, Z.~S., and Ramdas, A. (2023).
\newblock Adaptive privacy composition for accuracy-first mechanisms.
\newblock {\em arXiv preprint arXiv:2306.13824}.

\bibitem[Rogers et~al., 2016]{rogers2016privacy}
Rogers, R.~M., Roth, A., Ullman, J., and Vadhan, S. (2016).
\newblock Privacy odometers and filters: Pay-as-you-go composition.
\newblock {\em Advances in Neural Information Processing Systems}, 29.

\bibitem[Wang et~al., 2016]{Wang2016-lf}
Wang, Q., Zhang, Y., Lu, X., Wang, Z., Qin, Z., and Ren, K. (2016).
\newblock Real-time and spatio-temporal crowd-sourced social network data publishing with differential privacy.
\newblock {\em IEEE Trans. Dependable Secure Comput.}, pages 1--1.

\bibitem[Wang, 2023]{Wang_autodp}
Wang, Y.-X. (2023).
\newblock autodp: autodp: A flexible and easy-to-use package for differential privacy.

\bibitem[Whitehouse et~al., 2023]{Whitehouse2022-se}
Whitehouse, J., Ramdas, A., Rogers, R., and Wu, S. (2023).
\newblock Fully-adaptive composition in differential privacy.
\newblock In {\em International Conference on Machine Learning}, pages 36990--37007. PMLR.

\bibitem[Xu et~al., 2017]{Xu2017-on}
Xu, F., Tu, Z., Li, Y., Zhang, P., Fu, X., and Jin, D. (2017).
\newblock Trajectory recovery from ash: User privacy is not preserved in aggregated mobility data.
\newblock In {\em Proceedings of the 26th international conference on world wide web}, pages 1241--1250.

\bibitem[Zhang et~al., 2021]{Zhang2021-by}
Zhang, W., Krehbiel, S., Tuo, R., Mei, Y., and Cummings, R. (2021).
\newblock Single and multiple {Change-Point} detection with differential privacy.
\newblock {\em J. Mach. Learn. Res.}, 22(29):1--36.

\bibitem[Zhu and Wang, 2020]{zhu2020improving}
Zhu, Y. and Wang, Y.-X. (2020).
\newblock Improving sparse vector technique with renyi differential privacy.
\newblock {\em Advances in Neural Information Processing Systems}, 33:20249--20258.

\end{thebibliography}

\newpage
\appendix
%

\onecolumn

\section{PROOFS FROM SECTION 2}
Our main contribution relies on a number of proof techniques which are extended from the simpler Gaussian Report Noisy Max setting.

\subsection{Pure DP Gaussian Report Noisy Max}
\begin{restatable}[Pure DP for Gaussian Report Noisy Max]{theorem}{dpGNMAX}
\label{thm_expost_gnmax}

Let $M$ be the Gaussian Report Noisy Max mechanism with standard deviation $\sigma$ applied to $d > 1$ bounded queries $q_1, \ldots, q_d: \mathcal{X} \to [a, b]$, each with sensitivity bounded by $\Delta$. Let $c = b - a$. Then $M$ satisfies $\epsilon$-DP with
\begin{align}
\epsilon
&=
\frac{
    \mathbb{E}_{z \sim \mathcal{N}(0, 1)} \left [
        \Phi \left(z - \frac{c - 2 \Delta}{\sigma}\right)^{d-1}
        \right ]
            }
            {
    \mathbb{E}_{z \sim \mathcal{N}(0, 1)} \left [
        \Phi \left (
        z - \frac{
        c
        }{
        \sigma
        }
        \right )^{d-1}
        \right ]
            }
\enspace.
\end{align}
\end{restatable}

To prove our Pure DP Gaussian Report Noisy Max bound, we first analyze the effect of the sensitivity in the privacy loss of the Gaussian Report Noisy Max mechanism instantiated with uniformly bounded queries.

\begin{lemma}\label{lem:gnmax:sensitivity}
Let $M$ be the Gaussian Report Noisy Max mechanism with standard deviation $\sigma$ applied to $d > 1$ bounded queries $q_1, \ldots, q_d: \mathcal{X} \to [a, b]$, each with sensitivity bounded by $\Delta$. Then we have
\begin{align}
    \sup_{D \simeq D'} \sup_{o \in [d]}
    \frac{\Pr[M(D) = o]}{\Pr[M(D') = o]}
    &\leq
    \sup_{y_1, \ldots, y_{d-1} \in [a-b+2\Delta, b-a]}
    \frac{\EE{z \sim \cN(0,1)}{\prod_{i=1}^{d-1} \Phi\left(z - \frac{y_i - 2 \Delta}{\sigma}\right)}}{\EE{z \sim \cN(0,1)}{\prod_{i=1}^{d-1} \Phi\left(z - \frac{y_i}{\sigma}\right)}}
    \enspace.\label{eqn:gnmax:sensitivity}
\end{align}
\end{lemma}
\begin{proof}
Fix a pair of neighbouring datasets $D$ and $D'$. With a slight abuse of notation we let $q_i = q_i(D)$ and $q_i' = q_i(D')$ for all $i \in [d]$. Recall that $M(D) = \argmax_{i \in [d]} q_i + Z_i$ with $Z_i \sim \cN(0,\sigma^2)$, and $|q_i - q_i'| \leq \Delta$ for all $i$. Without loss of generality fix the output $o = d$. Then we have:
\begin{align}
    \Pr[M(D) = d]
    &=
    \Pr[\wedge_{i=1}^{d-1} q_d + Z_d > q_i + Z_i]
    \\
    &=
    \Pr[\wedge_{i=1}^{d-1} Z_i < Z_d + q_d - q_i]
    \\
    &=
    \Pr[\wedge_{i=1}^{d-1} Z_i < Z_d + q_d' - q_i' + (q_d - q_d') - (q_i - q_i')]
    \\
    &\leq
    \Pr[\wedge_{i=1}^{d-1} Z_i < Z_d + q_d' - q_i' + 2 \Delta]
    \\
    &=
    \EE{z \sim \cN(0,\sigma^2)}{\Pr[\wedge_{i=1}^{d-1} Z_i < z + q_d' - q_i' + 2 \Delta]}
    \\
    &=
    \EE{z \sim \cN(0,\sigma^2)}{\prod_{i=1}^{d-1} \Pr[Z_i < z + q_d' - q_i' + 2 \Delta]}
    \\
    &=
    \EE{z \sim \cN(0,\sigma^2)}{\prod_{i=1}^{d-1} \Phi\left(\frac{z + q_d' - q_i' + 2 \Delta}{\sigma}\right)}
    \\
    &=
    \EE{z \sim \cN(0,1)}{\prod_{i=1}^{d-1} \Phi\left(z + \frac{q_d' - q_i' + 2 \Delta}{\sigma}\right)}
    \enspace.
\end{align}
Therefore, writing $y_i = q_i' - q_d'$, we get
\begin{align}
    \frac{\Pr[M(D) = o]}{\Pr[M(D') = o]}
    &\leq
    \frac{\EE{z \sim \cN(0,1)}{\prod_{i=1}^{d-1} \Phi\left(z - \frac{y_i - 2 \Delta}{\sigma}\right)}}{\EE{z \sim \cN(0,1)}{\prod_{i=1}^{d-1} \Phi\left(z - \frac{y_i}{\sigma}\right)}} \enspace.
\end{align}

Note that a priori we have $q_i, q_i' \in [a,b]$ for all $i \in [d]$. However, for the above inequality to hold we set $q_i - q_i' = - \Delta$ for $i < d$ and $q_d - q_d' = \Delta$. These imply $q_i' = q_i + \Delta \in [a + \Delta, b]$ and $q_d' = q_d' - \Delta \in [a, b-\Delta]$, so $y_i = q_i' - q_d' \in [a - b + 2\Delta, b - a]$ for $i \in [d-1]$.
\end{proof}

Given the result above, all that remains is to identify where the supremum over the $y_i$ is attained in \Cref{eqn:gnmax:sensitivity}. To that end we will show that the ratio is non-decreasing in each of the $y_i$. The following H\^{o}pital-like monotonicity rule and property of ratios of moment generation functions will be useful.

\begin{lemma}[\citep{pinelis2001monotone, Anderson1993-wh}]
\label{thm_hospital_monotone}
Let $-\infty \leq a < b \leq \infty$ and let $f, g : [a, b] \to \mathbb{R}$ be continuous differentiable functions on $(a, b)$, with $f(a) = g(a) = 0$ or $f(b) = g(b) = 0$, and $g'(x) \neq 0$ for $x \in (a, b)$. If $f' / g'$ is non-decreasing in $(a,b)$, then so is $f / g$.
\end{lemma}

We say that a random variable $Z$ has \emph{tails majorized by an exponential decay} if the cumulative distribution function $F_Z$ of $Z$ is such that there exist positive constants $c_1, c_2$ satisfying $F_Z(z) = O(e^{c_1 z})$ for $z \to -\infty$ and $1-F_Z(z) = O(e^{-c_2 z})$ for $z \to \infty$. This is a well-known sufficient condition for the cumulant generating function $K_Z(s) = \log \E{e^{sZ}}$ to exist.

\begin{lemma}
\label{lem_mgf_increasing}
Suppose $Z$ is a random variable with tails majorized by an exponential decay. Then for any $t > 0$ the function
\begin{align}
    R(s) &= \frac{
        \mathbb{E} \left [ e^{(t + s)Z} \right]
    }{
        \mathbb{E} \left [ e^{sZ} \right]
    } \enspace,
\end{align}
is non-decreasing for all $s \in \mathbb{R}$.
\end{lemma}
\begin{proof}
Let $\xi(s) = \E{e^{s Z}}$ be the moment generating function of $Z$. Taking the derivative of $R$ we get:
\begin{align}
    R'(s) = \frac{
        \xi'(t + s) \xi(s) - \xi(t+s) \xi'(s)
    }{
        \xi(s)^{2}
    } \enspace.
\end{align}
Thus, $R'(s) \geq 0$ if and only if $\xi'(t + s) \xi(s) - \xi(t+s) \xi'(s) \geq 0$, which is equivalent to
\begin{align}
    \frac{\partial}{\partial s} \log \xi(t+s) = \frac{\xi'(t+s)}{\xi(t+s)} \geq \frac{\xi'(s)}{\xi(s)} = \frac{\partial}{\partial s} \log \xi(s) \enspace.
\end{align}
Therefore $R$ is non-decreasing if the derivative of the cumulant generating function $K_Z'(s)$ (which exists by assumption) is non-decreasing. But note that when $K_Z(s)$ exists it is known to be convex and infinitely differentiable, and therefore its derivative is non-decreasing.
\end{proof}

Note that by symmetry of the expression of interest in the $y_i$ it suffices to establish monotonicity with respect to a single variable. Thus we define the following functions:
\begin{align}
    f(x) &= \EE{z \sim \cN(t,1)}{\Phi\left(z - x \right) \prod_{i=1}^{k} \Phi\left(z - c_i \right)} \enspace, \\
    g(x) &= \EE{z \sim \cN(0,1)}{\Phi\left(z - x \right) \prod_{i=1}^{k} \Phi\left(z - c_i \right)} \enspace,
\end{align}
where $t > 0$ and $c_1, \ldots, c_{k} \in \R$ are constants.

\begin{lemma}\label{lem:gnmax:monotone}
The function $F(x) = f(x) / g(x)$ is non-decreasing for all $x \in \R$.
\end{lemma}
\begin{proof}
We are going to apply \Cref{thm_hospital_monotone}. First note that since $\lim_{x \to \infty} \Phi(z - x) = 0$ we have $\lim_{x \to \infty} f(x) = \lim_{x \to \infty} g(x) = 0$. Next we observe that, writing $\phi(u) = e^{-\frac{u^2}{2}} / \sqrt{2 \pi}$ for the density of a standard Gaussian random variable, we have
\begin{align}
    \frac{\partial}{\partial x} \Phi(z - x)
    &=
    \frac{\partial}{\partial x} \int_{-\infty}^{z-x} \phi(u) du
    =
    - \phi(z - x) \enspace.
\end{align}
Thus, when computing $f'(x)$ we will obtain a term of the form $\phi(z - t) \phi(z - x)$, which can be simplified to
\begin{align}
    \phi(z - t) \phi(z - x)
    &=
    \frac{1}{2\pi} e^{-\frac{(z-t)^2}{2}} e^{-\frac{(z-x)^2}{2}}
    =
    \frac{1}{2\pi} e^{-z^2} e^{z (x+t)} e^{-\frac{t^2 + x^2}{2}} \enspace.
\end{align}
From this we can now show that $f'(x)$ is proportional to the moment generating function of a random variable $Z$ with density $p(z) = e^{-z^2} \prod_{i=1}^{k} \Phi\left(z - c_i \right) / (2 \pi N)$ where $N$ is a normalizing constant (independent of $x$ and $t$):
\begin{align}
    f'(x) &= - \EE{z \sim \cN(t,1)}{\phi\left(z - x \right) \prod_{i=1}^{k} \Phi\left(z - c_i \right)}
    \\
    &=
    - \int_{-\infty}^{\infty} \phi(z-t) \phi(z-x) \prod_{i=1}^{k} \Phi\left(z - c_i \right) dz
    \\
    &=
    - e^{-\frac{t^2 + x^2}{2}} \cdot N \cdot \int_{-\infty}^{\infty} e^{z (x+t)} p(z) dz
    \\
    &=
    - e^{-\frac{t^2 + x^2}{2}} \cdot N \cdot \EE{Z \sim p}{e^{(x+t) Z}} \enspace.
\end{align}
A similar derivation also yields the following expression for $g'(x)$:
\begin{align}
    g'(x) &= - e^{-\frac{x^2}{2}} \cdot N \cdot \EE{Z \sim p}{e^{x Z}} \enspace.
\end{align}
Putting these two derivations together we obtain the following expression for the test function in \Cref{thm_hospital_monotone}:
\begin{align}
    H(x) := \frac{f'(x)}{g'(x)} = \frac{e^{-\frac{t^2}{2}} \EE{Z \sim p}{e^{(x+t) Z}}}{\EE{Z \sim p}{e^{x Z}}} \enspace.
\end{align}
Noting that the distribution with density $p(z)$ satisfies the condition of \Cref{lem_mgf_increasing} we see that $H$ is non-decreasing and therefore $F$ is non-decreasing.
\end{proof}

With all these ingredients in place we see that \Cref{thm_expost_gnmax} follows by using \Cref{lem:gnmax:monotone} to show that the supremum for each $y_i$ in \Cref{lem:gnmax:sensitivity} is attained at $y_i = b - a$.

\subsection{Pure DP Above Threshold}

\begin{algorithm}[H]
    \SetAlgoLined
    \SetKwInOut{Input}{input}
    \Input{dataset $D$; noise parameters $\sigma_X, \sigma_Z$; a stream of queries $q_1, q_2, \ldots$; threshold $\rho$.}
    
    $\hat{\rho} = \rho + \mathcal{N}(0, \sigma_X^2)$
    
    \For{$t = 1, 2, \ldots$}{
        
        $\hat{q_t} = q_t(D) + \mathcal{N}(0, \sigma_Z^2)$
        
        \uIf{$\hat{q_t} \geq \hat{\rho}$}{

            Output $x_t = \top$ and HALT
        }
        \Else{
        
            Output $x_t = \bot$
            
        }
    }
    \caption{Gaussian Above Threshold \citep{zhu2020improving}}
    \label{alg:Gaussian_at}
\end{algorithm}

\begin{restatable}[Pure Ex-post Gaussian Above Threshold]{theorem}{expostAT}

\label{thm_expost_at_f}
Given a stream $q_1, q_2, \ldots : \mathcal{X} \to [a,b]$ with global sensitivity $\Delta$, the Gaussian Above Threshold mechanism (Algorithm~\ref{alg:Gaussian_at}) satisfies $\epsilon_{\text{post}}$-ex-post-DP with

\begin{align}
    \epsilon_{\text{post}}(\{\bot^{t-1} \top \})  
        &= 
    \frac{
\mathbb{E}_{x \sim \mathcal{N}(0, 1)} \left [
    \Phi \left (\frac{\sigma_X x + \rho - (b - a) + \Delta}{\sigma_Z} \right )^{(t - 1)}
    \Phi \left ( \frac{- \sigma_X x - \rho + a + \Delta }{\sigma_Z} \right )
    \right ]
        }
        {
\mathbb{E}_{x \sim \mathcal{N}(0, 1)} \left [
    \Phi \left (\frac{\sigma_X x + \rho - (b - a)}{\sigma_Z} \right )^{(t - 1)}
    \Phi \left ( \frac{- \sigma_X x - \rho + a }{\sigma_Z} \right )
    \right ]
        }
\enspace.
\end{align}

\end{restatable}

Overall, our proof follows a similar strategy to the proof in the previous section.

\begin{lemma}\label{lem:svt:sensitivity}
Let $M$ be the Gaussian Above Threshold mechanism, with public threshold $\rho \geq 0$, threshold noise of standard deviation $\sigma_X$, and query noise of standard deviation $\sigma_Z$ applied to a stream of bounded queries $q_1, q_2, \ldots : \mathcal{X} \to [a,b]$. Then for any stopping time $t \geq 1$ we have
    \begin{align}
    \sup_{D \simeq D'} \frac{\Pr[M(D) = t]}{\Pr[M(D') = t]}
    \leq
    \sup_{y_1, \ldots, y_{t-1} \in [a+\Delta,b], y_t \in [a, b-\Delta]}
    \frac{\EE{x \sim \cN(0, 1)}{\prod_{i=1}^{t-1} \Phi\left(\frac{\sigma_X x + \rho - y_i + \Delta}{\sigma_Z}\right) \cdot \Phi\left(\frac{- \sigma_X x - \rho + y_t + \Delta}{\sigma_Z}\right)}}{\EE{x \sim \cN(0, 1)}{\prod_{i=1}^{t-1} \Phi\left(\frac{\sigma_X x + \rho - y_i }{\sigma_Z}\right) \cdot \Phi\left(\frac{- \sigma_X x - \rho + y_t}{\sigma_Z}\right)}}
    \enspace.
    \end{align}
\end{lemma}
\begin{proof}
Fix a pair of datasets $D \simeq D'$. Like before, we let $q_t = q_t(D)$ and $q_t' = q_t(D')$; they satisfy $|q_t - q_t'| \leq \Delta$.
First of all we bound the probability that the mechanism on input $D$ stops at time $t$ as follows:
\begin{align}
    \Pr[M(D) = t] 
    &= 
    \Pr[\wedge_{i=1}^{t-1} \{ q_i + Z_i  < X + \rho \} \wedge \{ X + \rho < q_{t} + Z_t \}] \\
     &= 
    \Pr[\wedge_{i=1}^{t-1} \{ Z_i  < X + \rho - q_i \} \wedge \{ X + \rho - q_t < Z_t \}] \\
    &=
    \Pr[\wedge_{i=1}^{t-1} \{ Z_i  < X + \rho - q_i' - (q_i - q_i') \} \wedge \{ X + \rho - q'_t - (q_t - q_t') < Z_t \}] \\
    &\leq
    \Pr[\wedge_{i=1}^{t-1} \{ Z_i  < X + \rho - q_i' + \Delta \} \wedge \{ X + \rho - q'_t - \Delta < Z_t \}] \\
    &=
    \EE{x \sim \cN(0, \sigma_X^2)}{\Pr[\wedge_{i=1}^{t-1} \{ Z_i  < x + \rho - q_i' + \Delta \} \wedge \{ x + \rho - q'_t - \Delta < Z_t \}]}
    \\
    &=
    \EE{x \sim \cN(0, \sigma_X^2)}{\prod_{i=1}^{t-1} \Pr[Z_i  < x + \rho - q_i' + \Delta] \cdot \Pr[x + \rho - q'_t - \Delta < Z_t]} \\
    &=
    \EE{x \sim \cN(0, \sigma_X^2)}{\prod_{i=1}^{t-1} \Pr[Z_i  < x + \rho - q_i' + \Delta] \cdot \Pr[-Z_t < - x - \rho + q'_t + \Delta]} \\
    &=
    \EE{x \sim \cN(0, \sigma_X^2)}{\prod_{i=1}^{t-1} \Pr[Z_i  < x + \rho - q_i' + \Delta] \cdot \Pr[Z_t < - x - \rho + q'_t + \Delta]} \\
    &=
    \EE{x \sim \cN(0, \sigma_X^2)}{\prod_{i=1}^{t-1} \Phi\left(\frac{x + \rho - q_i' + \Delta}{\sigma_Z}\right) \cdot \Phi\left(\frac{- x - \rho + q'_t + \Delta}{\sigma_Z}\right)} \\
    &=
    \EE{x \sim \cN(0, 1)}{\prod_{i=1}^{t-1} \Phi\left(\frac{\sigma_X x + \rho - q_i' + \Delta}{\sigma_Z}\right) \cdot \Phi\left(\frac{- \sigma_X x - \rho + q'_t + \Delta}{\sigma_Z}\right)}
    \enspace.
\end{align}
A similar derivation also yields:
\begin{align}
    \Pr[M(D') = t]
    &=
    \EE{x \sim \cN(0, 1)}{\prod_{i=1}^{t-1} \Phi\left(\frac{\sigma_X x + \rho - q_i'}{\sigma_Z}\right) \cdot \Phi\left(\frac{- \sigma_X x - \rho + q'_t}{\sigma_Z}\right)}
    \enspace.
\end{align}
Thus, the ratio of probabilities of $M$ stopping at time $t$ on $D$ and $D'$ can be bounded as:
\begin{align}
    \frac{\Pr[M(D) = t]}{\Pr[M(D') = t]}
    \leq
    \frac{\EE{x \sim \cN(0, 1)}{\prod_{i=1}^{t-1} \Phi\left(\frac{\sigma_X x + \rho - q_i' + \Delta}{\sigma_Z}\right) \cdot \Phi\left(\frac{- \sigma_X x - \rho + q_t' + \Delta}{\sigma_Z}\right)}}{\EE{x \sim \cN(0, 1)}{\prod_{i=1}^{t-1} \Phi\left(\frac{\sigma_X x + \rho - q_i' }{\sigma_Z}\right) \cdot \Phi\left(\frac{- \sigma_X x - \rho + q_t'}{\sigma_Z}\right)}}
    \enspace.
\end{align}
Now note that in the bound we set $q_i - q_i' = -\Delta$ for $i < t$ and $q_t - q_t' = \Delta$. Since all the queries are bounded in $[a,b]$ we have that $y_i = q_i' = q_i + \Delta \in [a+\Delta, b]$ for $i < t$ and $y_t = q_t' = q_t - \Delta \in [a, b-\Delta]$. Taking the supremum over these ranges completes the proof.
\end{proof}

To conclude, we show that the supremum in \Cref{lem:svt:sensitivity} is attained at $y_1 = \cdots = y_{t-1} = b$ and $y_t = a$. This follows from observing that the ratio is increasing in $y_1, \ldots, y_{t-1}$ and decreasing in $y_t$, which follows from  the same argument used in Lemma~\ref{lem:gnmax:monotone}.

\newpage
\subsection{Filtered Self-Reporting Composition}
\begin{restatable}{theorem}{fsrcEpsDP}
\label{thm:fsrc_algorithm_dp}

Suppose the mechanisms $M_1, M_2, \ldots$ provided to 
\Cref{alg:fsrc_meta} are such $M_t$ is $(\epsilon_{\max,t}, \delta)$-pDP and $\epsilon_{\text{post},t}$-ex-post-DP for all $t$.
Then \Cref{alg:fsrc_meta} satisfies $(\epsilon, \delta)$-DP.
\end{restatable}

\begin{algorithm}[H]
\SetAlgoLined
\SetKwInOut{Input}{input}
\Input{dataset $D$; privacy budget $\epsilon$; a stream of adaptive mechanisms $M_1, M_2, \ldots$; a stream of values $\epsilon_{\max,1}, \epsilon_{\max, 2}, \ldots$; a stream of functions $\epsilon_{\text{post},1}, \epsilon_{\text{post},2}, \ldots$}
\For{$t = 1, 2, \ldots$}{
    \uIf{$\sum_{i=1}^{t-1} \epsilon_i + \epsilon_{\max,t} \geq \epsilon$}{
        HALT
    }
    \Else{
        $o_t = M_t(D; o_{1:t-1})$
        
        $\eps_t = \epsilon_{\text{post},t}(o_t)$
        
        Release $o_t$
    }
}
\caption{Filtered Composition, Ex-Post Privacy}
\label{alg:fsrc_meta}
\end{algorithm}

\begin{proof}
We will prove that the mechanism satisfies $(\epsilon, \delta)$-pDP and therefore also $(\epsilon,\delta)$-DP. Let $M$ denote the mechanism and fix a pair of neighboring datasets $D$ and $D'$. Let $o = o_1, \ldots, o_T \sim M(D)$ be sampled from the mechanism (here $T$ is a random stopping time) and consider the privacy loss random variable
\begin{align*}
    \mathcal{L}(o) &= \log \frac{\Pr[M(D)=o]}{\Pr[M(D')=o]}
    \\
    &=
    \sum_{i=1}^T \log \frac{\Pr[M_i(D; o_{1:i-1}) = o_i]}{\Pr[M_i(D'; o_{1:i-1}) = o_i]}
    +
    \log \frac{\Pr[\text{$M(D)$ halts after outputting $o_{1:T}$}]}{\Pr[\text{$M(D')$ halts after outputting $o_{1:T}$}]} \enspace.
\end{align*}
First of all we observe that given the outputs up to a certain time step $t$, the stopping condition is deterministic and independent of the dataset. Thus, the last term above vanishes.
Furthermore, since each of the mechanisms $M_i$ satisfies $\epsilon_{\text{post},i}$-ex-post-DP, for all $i \in [T-1]$ we have
\begin{align*}
    \log \frac{\Pr[M_i(D; o_{1:i-1}) = o_i]}{\Pr[M_i(D'; o_{1:i-1}) = o_i]} \leq \epsilon_{\text{post},i}(o_i) \enspace.
\end{align*}
In addition, since $M_T$ is $(\epsilon_{\max,T}, \delta)$-pDP, we also have
\begin{align*}
    \Pr\left[ \log \frac{\Pr[M_T(D; o_{1:T-1}) = o_T]}{\Pr[M_T(D'; o_{1:T-1}) = o_T]} > \epsilon_{\max,T} \right] \leq \delta \enspace.
\end{align*}
Now, since the stopping rule guarantees that $\sum_{i=1}^{T-1} \epsilon_{\text{post},i}(o_i) + \epsilon_{\max,T} \leq \epsilon$, we get
\begin{align*}
    \Pr[ \mathcal{L}(o) > \epsilon]
    &\leq
    \Pr\left[ \sum_{i=1}^{t-1} \epsilon_{\text{post},i}(o_i) + \log \frac{\Pr[M_T(D; o_{1:T-1}) = o_T]}{\Pr[M_T(D'; o_{1:T-1}) = o_T]} > \epsilon \right] < \delta \enspace.
\end{align*}

\end{proof}

\newpage 

\section{PURE DP OFFLINE SELECTION MECHANISMS}
In \cref{fig:gnmax-num-estimate-spread} we illustrate the variance observed when taking a Monte-Carlo estimate over a range of values $\sigma$ when computing our Pure DP Gaussian Report Noisy Max bound.
\begin{figure}[H]
    \centering
    \includegraphics[width=1.\linewidth]{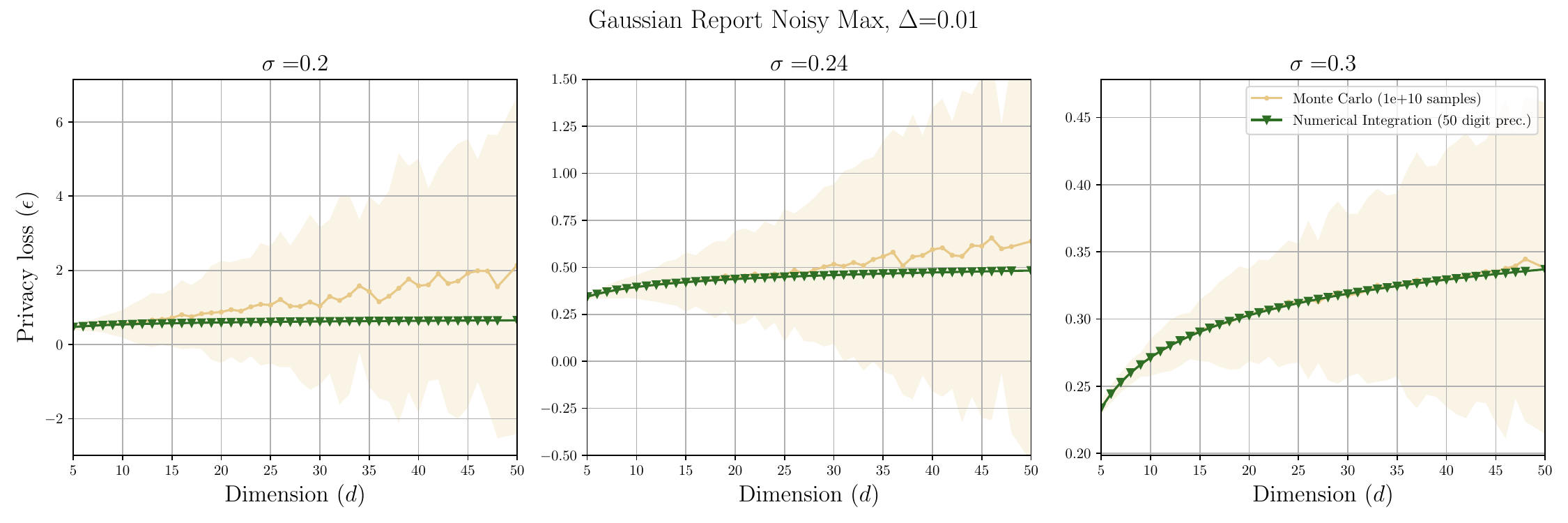}
    \caption{
    Comparison of numerical integration (in green) and Monte Carlo estimation (in beige, with 100 trials, each with 10B samples) are reported. Note Monte Carlo sampling methods are much slower than numerical methods and require greater numbers of samples as the dimension increases. We observe higher variance in the privacy loss estimate as dimension increases. The shaded region represents the standard deviation.}
    \label{fig:gnmax-num-estimate-spread}
\end{figure}

\subsection{Numerical evaluation compared to post-processing bounds}
We study how privacy loss changes as noise (and privacy) are increased in \cref{fig:num_estimate_dim_loss} for Gaussian Report Noisy Max. The standard deviation of the query noise, $\sigma$, and the number of queries (dimension $d$) on the privacy loss estimate. We note a slow increase in the privacy loss as the number of queries increases, in particular when compared to the ex-ante analysis with $\Delta_q = \Delta \sqrt{d}$.

\begin{figure}[H]
    \centering
    \includegraphics[width=1\linewidth]{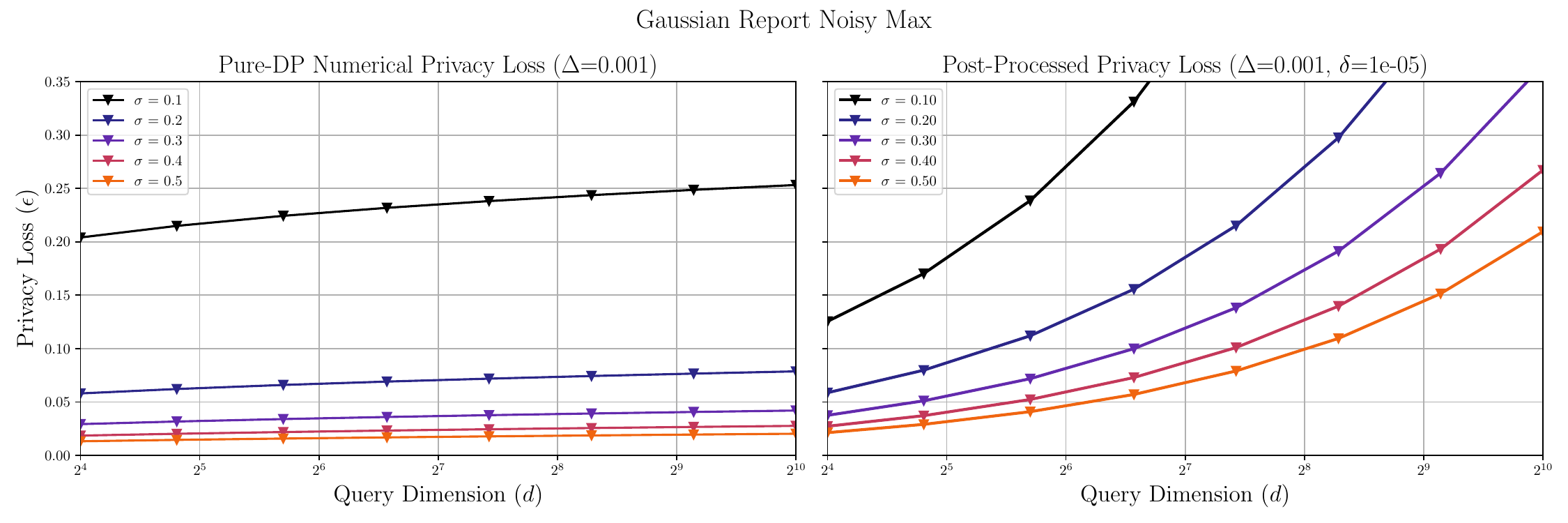}
    \caption{Left: numerical evaluation of the privacy loss for Gaussian Report Noisy Max. Privacy loss increase is very slowly as the number of queries increases. Right: we calculate privacy loss with a classic $(\epsilon, \delta)$ post-processing bound. }
    \label{fig:num_estimate_dim_loss}
\end{figure}

\newpage
\subsection{Where pure DP offers tighter accounting for Report Noisy Max}
To assess whether there are any utility gains from performing a numerical evaluation of the privacy loss, we produce a heatmap of the difference reported between the standard post-processing bound under RDP and our method in \cref{fig:gnmax-heatmap}. There exists a smooth region where the privacy accounting difference is significant, particularly in the the high privacy, high query setting. As the queries become less sensitive ($\Delta \to 0$), this region becomes more pronounced.

\begin{figure}[H]
    \centering
    \includegraphics[width=1\linewidth]{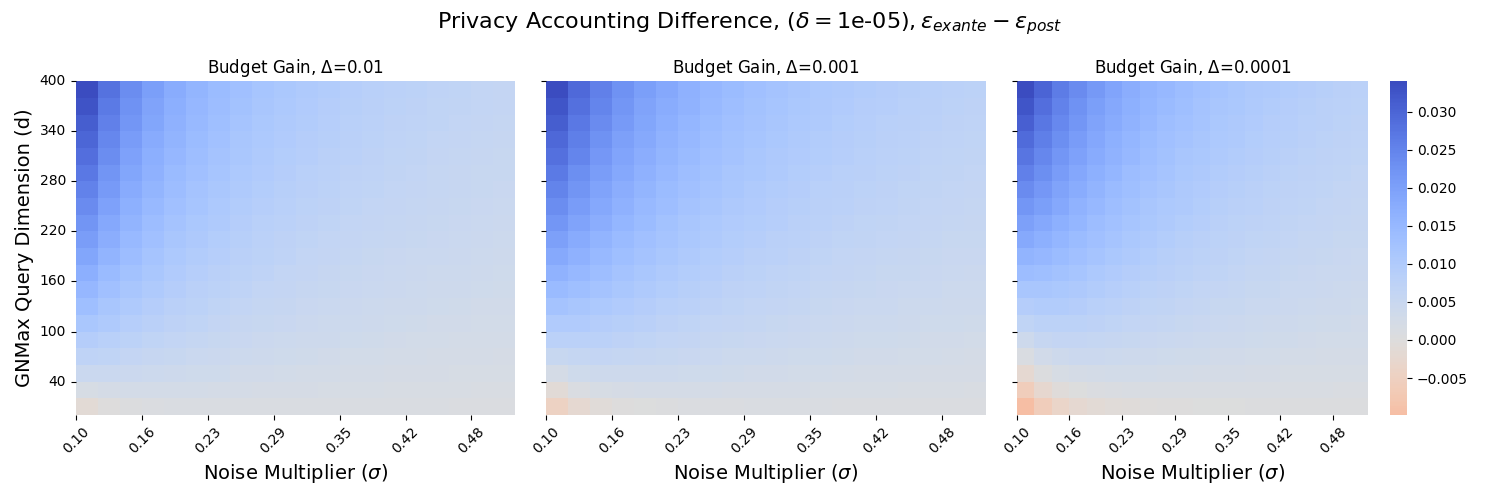}
    \caption{Comparison of the reported privacy loss difference between the standard Gaussian Report Noisy Max analysis and an ex-post evaluation for different levels of privacy ($\sigma$) and number of queries ($d$). For large $d$ and small $\sigma$, we observe the greatest difference in reported privacy loss.}
    \label{fig:gnmax-heatmap}
\end{figure}

\subsection{Offline Selection with Gaussian Report Noisy Max}
 Each accuracy/privacy loss plot reports the mean value over 1,000 trials, with the shaded region covering the standard error. 

In our experiments, we normalize the dataset to values between zero and one. The shaded region represents the standard error across several runs. The mean value is represented in as a line. A classic example of this sort of problem is query selection over a long time horizon. Our experiments center on energy and mobility datasets, which have been known to leak privacy  \citep{Narayanan2008-hw}. In both instances, user behavior, such as whether power consumption deviates from historical levels, can be joined with other datasets to infer changes to a person's socio-economic situation. 

\paragraph{Dataset Pre-Processing} The datasets we consider are temporal event logs over a fixed period of time with a regular reporting interval (e.g. hour, day, month). In each instance, we re-scale the reported values to be between zero and one. These transformations do not affects the task result since we are interesting in reporting a discrete value. In the offline selection task, we wish to find the time step whose reported value is largest, and in the online selection task, the goal is to produce a binary vector indicating which step exceeds a fixed threshold.

\paragraph{Utility} We measure the accuracy of Gaussian Report Noisy Max by comparing the distance between the true answer and the noisy (private) answer step. Let $q_i^*$ be the true query answer and $\hat{q}_i$ be the answer selected by Report Noisy Max. Then accuracy $\text{Acc}_{\sigma}$ is defined as
\begin{align}
  \text{Acc}_{\sigma}(D) = 1 - \mathbb{E} \left [ \frac{| q_{i^{*}}(D) - q_{\hat{i}}(D) |}{q_{\text{max}}} \right ] \enspace.
\end{align}
where $q_{\text{max}}$ is the maximum value of any query. In our experiments, $q_{\text{max}} = 1$.
In this construction, perfect utility occurs when $\text{Acc}_{\sigma}(D) = 1$. This is a reasonable measure of utility since we wish to measure how close on average we are to the max query value.

\newpage
\paragraph{NEAR Energy Dataset.} The NEAR Energy Dataset \citep{neardataset} includes the annual energy consumption of 300 customers, we formulate the following question: Which day was consumption highest?

\begin{figure}[H]
    \centering
    \includegraphics[width=0.7\linewidth]{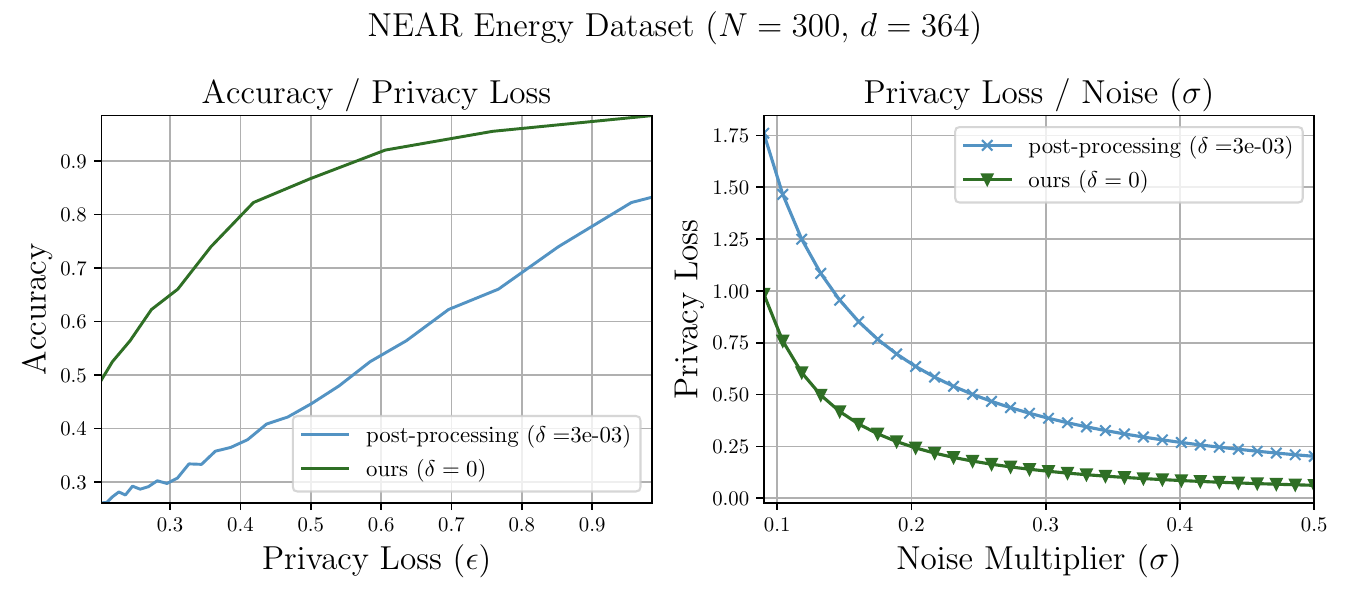}
    \caption{Privacy / utility when performing Gaussian Report Noisy Max on a query over $d=364$. $90\%$ Accuracy is attainable with our method whereas the same would only be possible with over double the privacy budget. }
    \label{fig:near365}
\end{figure}
\vspace{-1.5\baselineskip}
\paragraph{London Energy Dataset}
Here we follow the same analysis as in the case of the NEAR dataset
\citep{lcl_smartmeter_london_energy_data}, however the dataset comprises of $N= 5,564$ customers over $d = 829$ days. The larger number of queries \emph{increases} the privacy cost, however this is balanced by a \emph{decrease} in individual contribution to each query due to each query having $\Delta = 1/N$.
\vspace{-0.5\baselineskip}
\begin{figure}[H]
    \centering
    \includegraphics[width=0.7\linewidth]{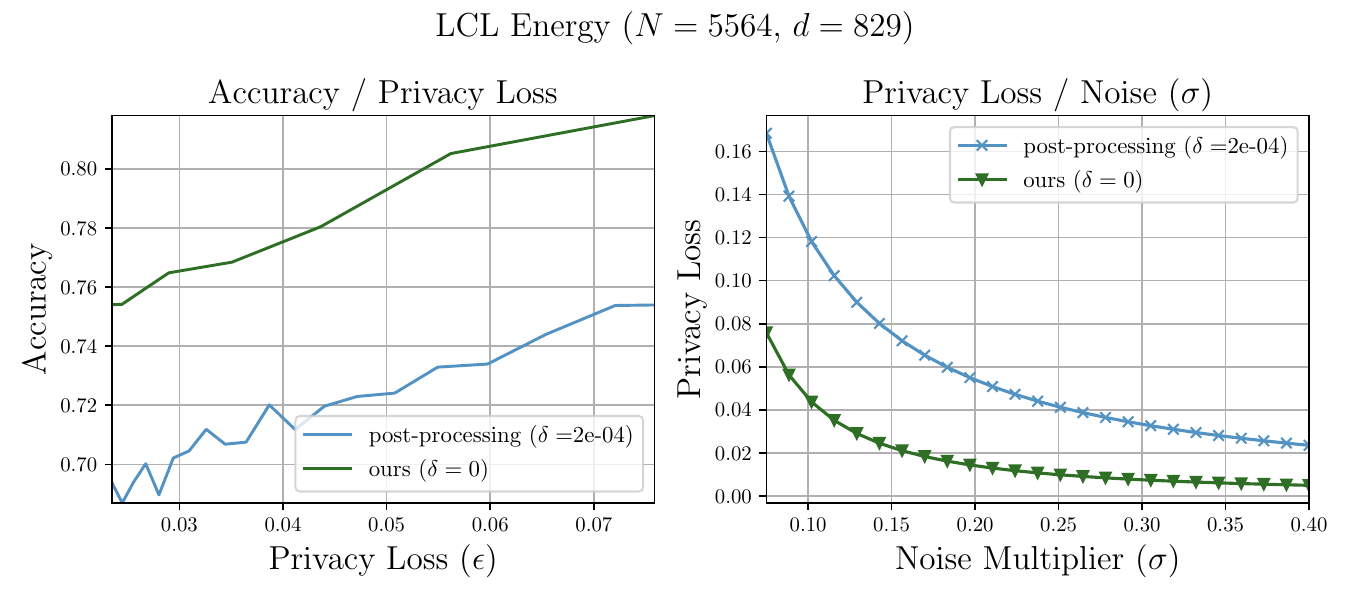}
    \caption{Privacy / utility when performing Gaussian Report Noisy Max on a query over $d=829$ with the London Energy Dataset. As $N$ and $d$ increase, the difference in reported privacy loss widens.}
    \label{fig:lcl829}
\end{figure}
\vspace{-0.5\baselineskip}
\begin{figure}[H]
    \centering
    \includegraphics[width=0.7\linewidth]{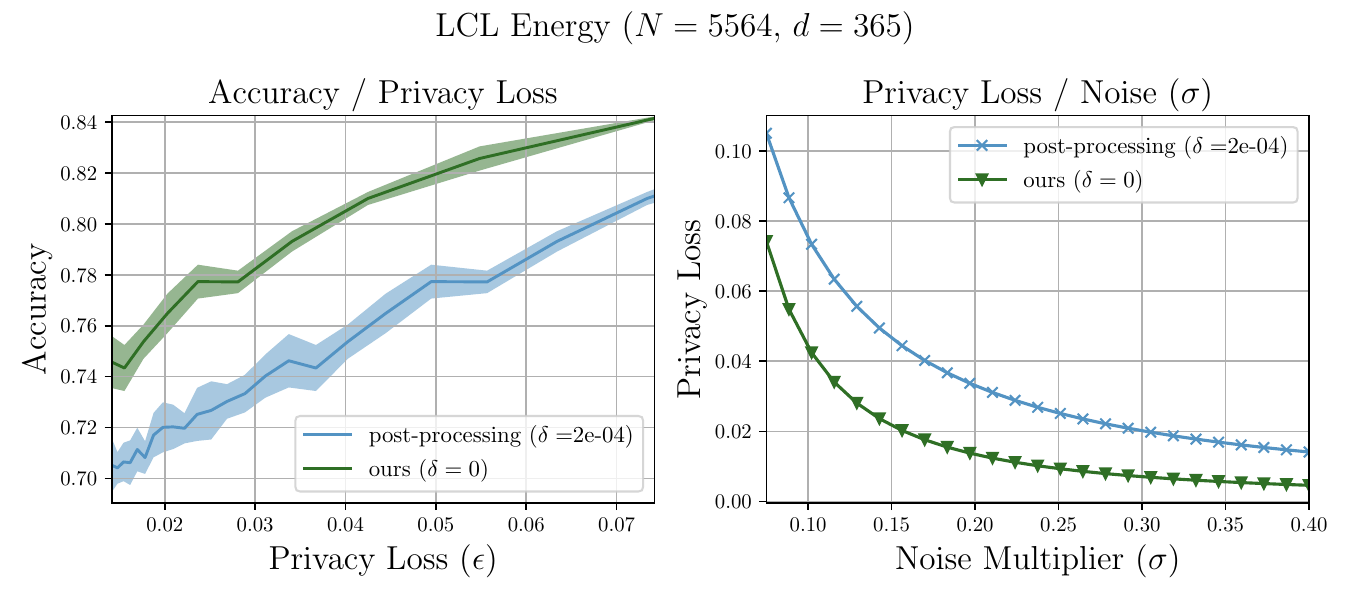}
    \caption{Privacy / utility when performing Gaussian Report Noisy Max on a query over $d=365$ with the London Energy Dataset.}
    \label{fig:lcl365}
\end{figure}

We remark that there are other offline selection mechanisms that guarantee pure DP without Gaussian noise. The exponential mechanism, which calibrates an exponential distribution with respect to a utility function, is commonly used solution to the offline query selection problem \citep{Dwork2014}. Despite a relatively simple implementation, the exponential mechanism requires that a sum over all query values be computed to fix the probability distribution. Permute-and-flip is a drop-in replacement for the exponential mechanism which has been shown to provide some utility benefits, but has the same calibration requirement \citep{mckenna2020permute}. In both instances the global sensitivity is no longer dependent on the number of queries since the maximum contribution for each user is taken with respect to each query individually.

In \cref{fig:energy_exp} we measure compare our method to pure DP baselines with other noise distributions. The global sensitivity for these exponential-based mechanisms is $\Delta := 1/N$, where $N$ is the number of users, compared to our bound, which depends on $d$ queries. One could make apply the same post-processing argument for Gaussian Report Noisy Max with noise drawn from the Laplace distribution however, this approach was omitted from our baselines since it did not appear competitive in our setting. Finally, we remark that our method may be easier to calibrate since each noisy query can be computed separately.
\vspace{-0.5\baselineskip}
\begin{figure}[H]
    \centering
    \begin{subfigure}{0.4\linewidth}
        \includegraphics[width=\linewidth]{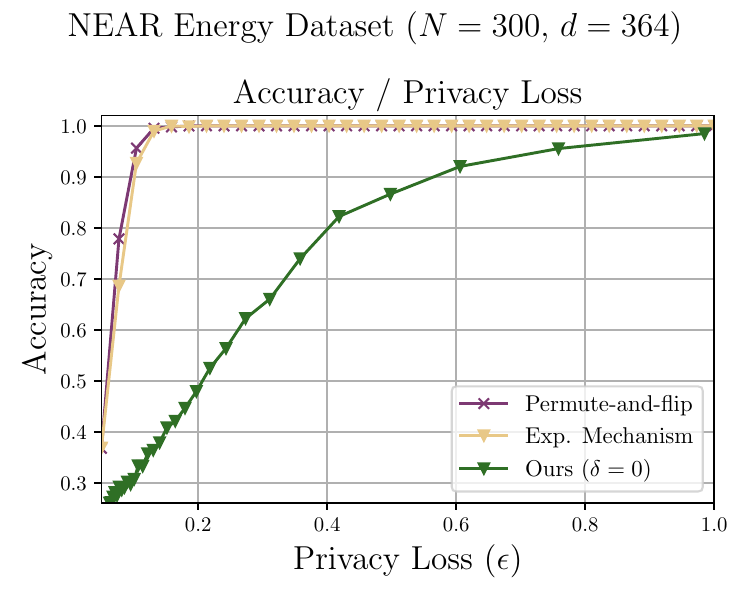}
        \caption{Comparison for $d=364$ on the NEAR Dataset.}
        \label{fig:energy_exp.1}
    \end{subfigure}
    \hfill
    \begin{subfigure}{0.4\linewidth}
        \includegraphics[width=\linewidth]{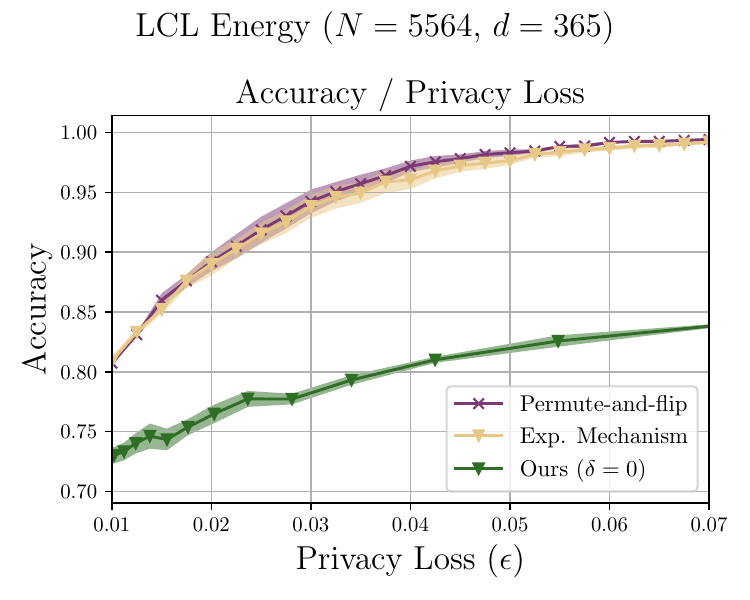}
        \caption{Comparison for $d=365$ on the LCL Energy Dataset.}
        \label{fig:energy_exp.3}
    \end{subfigure}
    \caption{Comparison between our method and two classical offline selection mechanisms.}
    \label{fig:energy_exp}
\end{figure}
\vspace{-0.5\baselineskip}
\paragraph{UCI Bikes Dataset}
\label{sec:orga51b5cd}
The UCI Bike Sharing Dataset captures the utilization of shared bikes in a geographic area over the course of a year \citep{ucibikesdataset}.
We apply the same definition of accuracy and report on the day with peak consumption. Since we do not know the upper bound on registered customers, we take the maximum (6,946 users) and assume that this is a public value. We note that our analysis is still worst-case, in that a user is assumed to be contributing to each daily (normalized) count query. We limit our Report Noisy Max query to 365 days.
\vspace{-0.5\baselineskip}
\begin{figure}[H]
    \centering
    \includegraphics[width=0.7\linewidth]{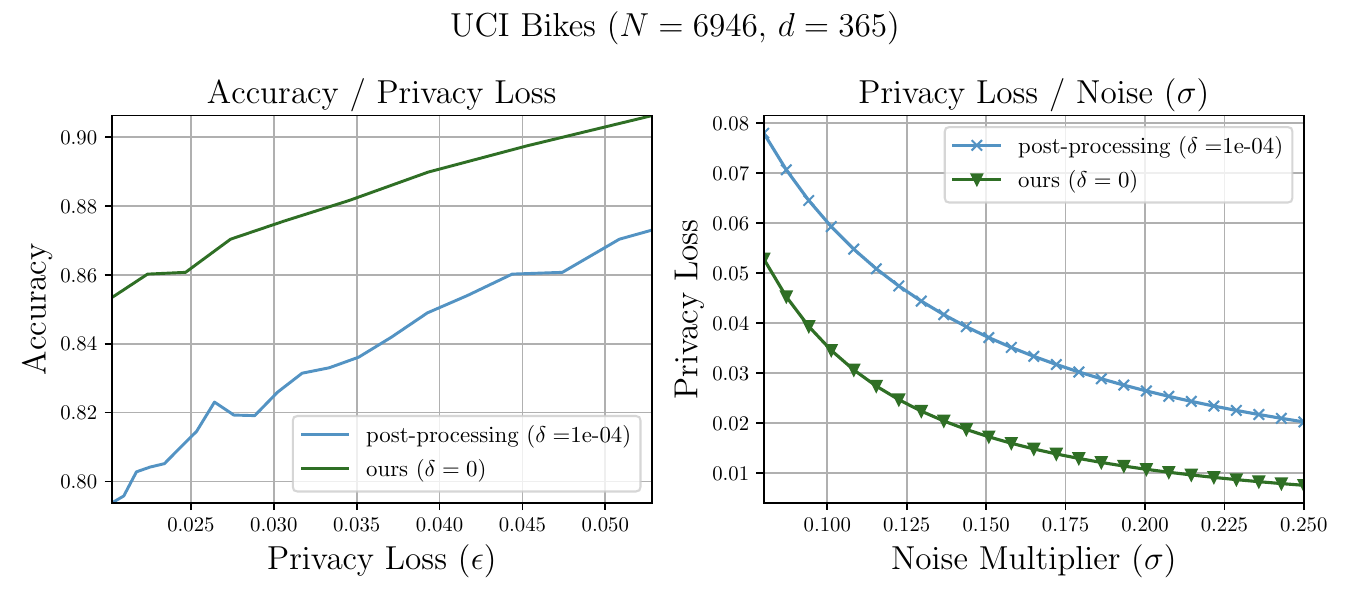}
    \caption{Privacy / utility when performing Gaussian Report Noisy Max on a query over $d=365$ with the UCI Bike Sharing Dataset.}
    \label{fig:uci1}
\end{figure}
\vspace{-1.5\baselineskip}
\begin{figure}[H]
    \centering
    \includegraphics[width=0.7\linewidth]{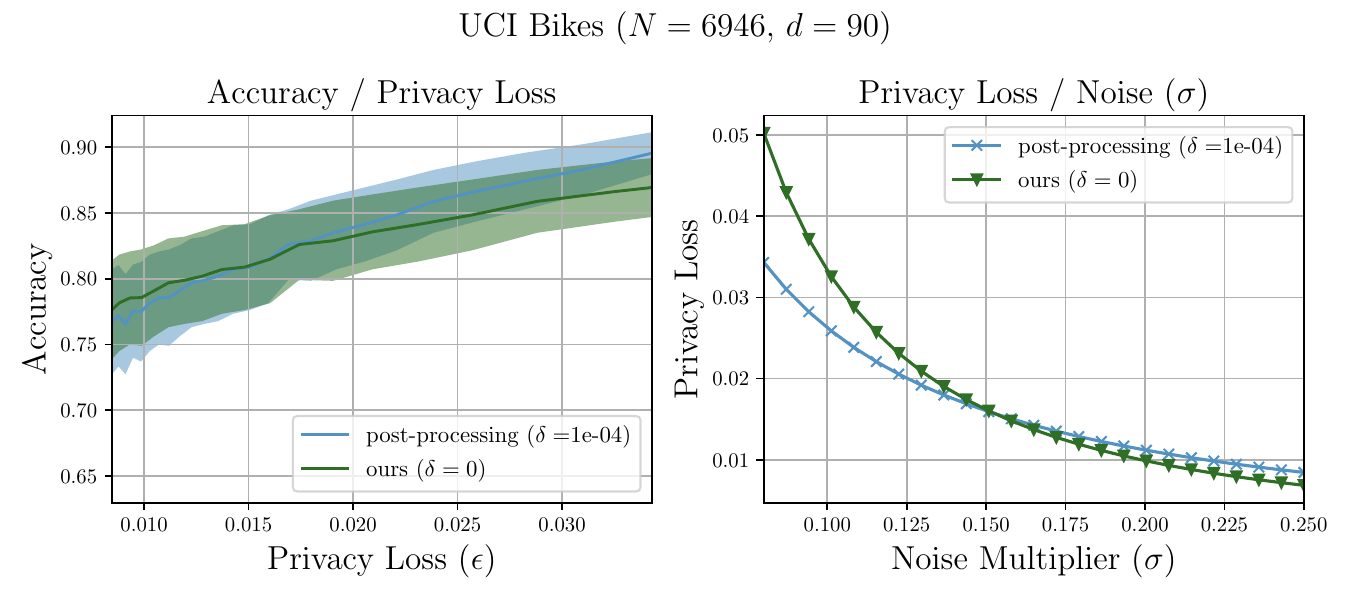}
    \caption{Privacy / utility when performing Gaussian Report Noisy Max on a query over $d=90$ with the UCI Bike Sharing Dataset. Note that as the number of queries, $d$, decreases, baseline methods become more competitive.}   
    \label{fig:uci90}
\end{figure}
\vspace{-1.5\baselineskip}
\begin{figure}[H]
    \centering
    \begin{subfigure}{0.4\linewidth}
        \centering
        \includegraphics[width=\linewidth]{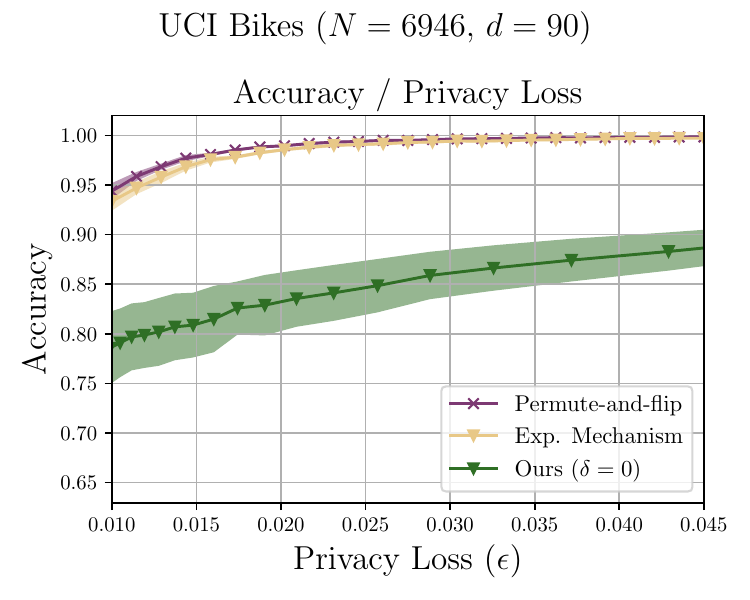}
        \caption{Comparison for $d=90$.}
        \label{fig:uci_exp.2}
    \end{subfigure}
    \hfill
    \begin{subfigure}{0.4\linewidth}
        \centering
        \includegraphics[width=\linewidth]{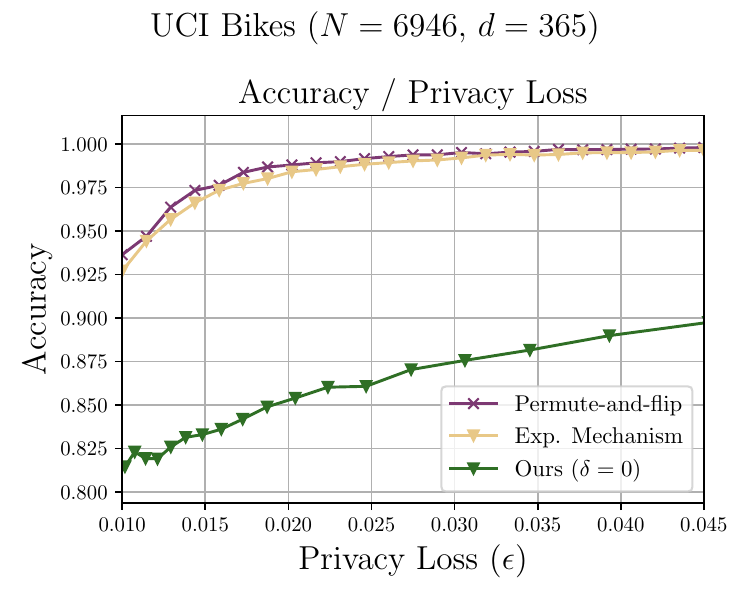}
        \caption{Comparison for $d=365$.}
        \label{fig:uci_exp.1}
    \end{subfigure}
    \caption{Comparison between our method and two classical offline selection mechanisms.}
    \label{fig:uci_exp}
\end{figure}
As shown in \cref{fig:uci_exp}, the Exponential Mechanism and Permute-and-flip offer a competitive baseline. 
\cref{fig:uci_exp_lap} provides a comparison of our method with Laplace Report Noisy Max. We observe a clear improvement in the pure DP setting using Gaussian noise and our bounds. 
\begin{figure}[H]
    \centering
    \begin{subfigure}{0.4\linewidth}
        \centering
        \includegraphics[width=\linewidth]{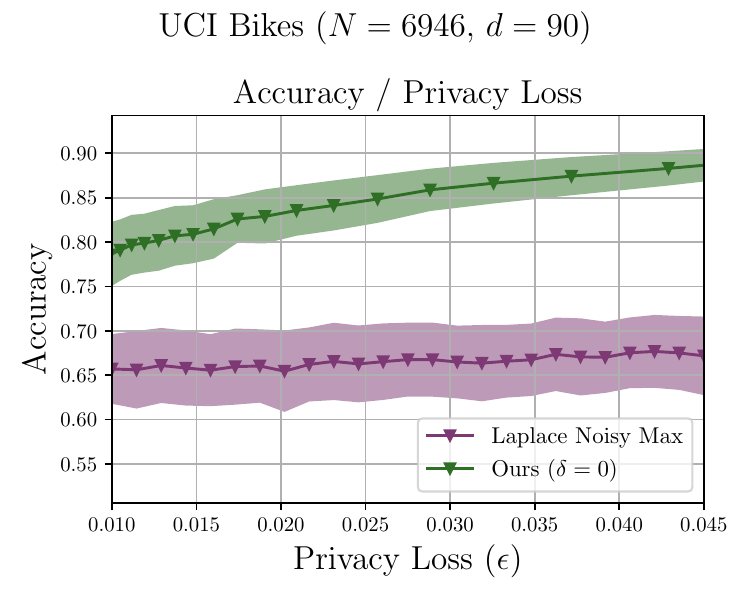}
        \caption{Comparison for $d=90$.}
        \label{fig:uci_lap.2}
    \end{subfigure}
    \hfill
    \begin{subfigure}{0.4\linewidth}
        \centering
        \includegraphics[width=\linewidth]{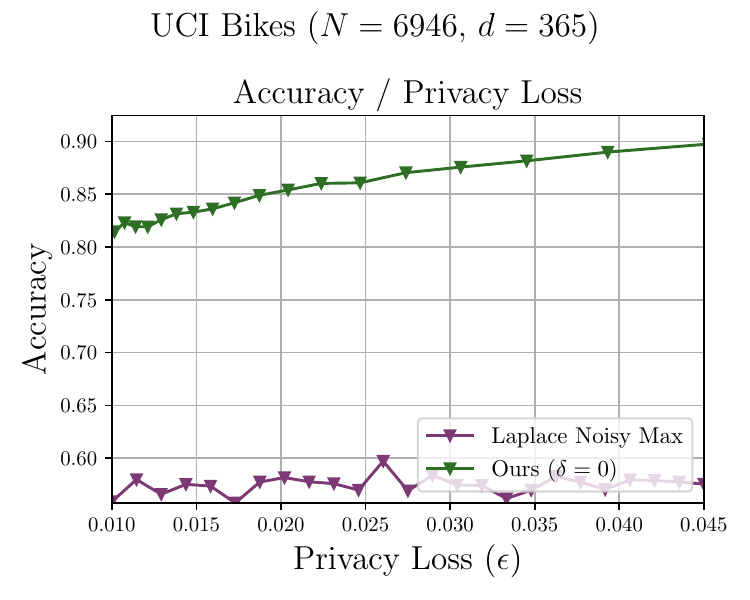}
        \caption{Comparison for $d=365$.}
        \label{fig:uci_lap.1}
    \end{subfigure}
    \caption{Comparison between our method and the Laplace Report Noisy Max on the UCI Bike Sharing Dataset.}
    \label{fig:uci_exp_lap}
\end{figure}

\section{COMPARISON WITH PRIVACY FILTERS}
Our baseline for comparison are recent results by \citep{Rogers2023-uv}, who propose a privacy filter which introduces an additive bound in both $\epsilon$ and $\delta$, thereby creating a result similar to advanced composition. Their construction does not rely on computing some $\epsilon_{\text{max}}$. 

\begin{theorem}[$(\epsilon, \delta)$-DP Filters \citep{Whitehouse2022-se}]
    Suppose $(M_t) \geq 1$ is a sequence of mechanisms such that, for any $t \geq 1$, $M_t$ is $(\epsilon_t, \delta_t)$-differentially private conditioned on $M_{1:t-1}$. Let $\epsilon > 0$, and $\delta = \delta' + \delta''$ be max privacy parameters s.t. $\delta' > 0, \delta'' \geq 0$. Let the stopping time function $N : \mathbb{R}_{\geq 0}^{\infty} \times \mathbb{R}_{\geq 0}^{\infty} \to \mathbb{N} $ be given by,
    \begin{align}
        N((\epsilon_t)_{t \geq 1}, (\delta_t)_{t \geq 1}) &= 
        \inf \left\{n : \epsilon<\sqrt{2 \log \left(\frac{1}{\delta^{\prime}}\right) \sum_{m \leq t+1} \epsilon_m^2}+\frac{1}{2} \sum_{m \leq t+1} \epsilon_m^2 \quad \text { or } \quad \delta^{\prime \prime}<\sum_{m \leq t+1} \delta_m\right\} \enspace .
    \end{align}
    Then $M_{1:N(\cdot)}(\cdot) : \mathcal{X} \to \mathcal{O}^{\infty} $ is $(\epsilon, \delta)$-DP.
\end{theorem}

To compare Filtered Self-Reporting Composition, we apply the following privacy filter, (\cref{def:rog_dp_filter}) with \cref{alg:comp_privacy_filter}. The bound we compute for Gaussian Above Threshold comes from \citep{zhu2020improving}.

\begin{definition}[DP Privacy Filter \citep{rogers2016privacy}]
\label{def:rog_dp_filter}
    Let $N(\cdot, \cdot)$ be a $(\epsilon, \delta)$-DP Filter. Then the DP Composition Privacy Filter $\cal{F}(\cdot)$ is given by 
    \begin{align}
        \mathcal{F}_{\epsilon, \delta}((\epsilon_{t})_{t \geq 1}, (\delta_t)_{t \geq 1}) &= 
        \begin{cases}
             \text{HALT} &  \text{ if } N((\eps_t)_{t \geq 1}, (\delta_t)_{t \geq 1}) > t \\
             \text{CONT.} &  \text{ otherwise }
        \end{cases} \enspace .
    \end{align}
\end{definition}

 \begin{algorithm}[H]
    \SetAlgoLined
    \SetKwInOut{Input}{input}
    \Input{Dataset $D$, Privacy budget $\epsilon_{\text{max}}$, $\delta_{\text{max}}$, $(\eps_t, \delta_t)$-differentially private mechanisms $M_t$, for $t=1, \ldots, T$, Privacy Filter $\cal{F}_{\eps, \delta}$.}
    
    \For{$t = 1, \ldots, T$}{
        $o_t = M_t(D)$ 
        
        \uIf{${\cal{F}}_{\epsilon_{\text{max}}, \delta_{\text{max}}}((\eps_t)_{t \geq 1}, (\delta)_{t \geq 1}) = \text{HALT}$}{
        
            $v_t = \bot$
            
        }
        \Else{
            
            $v_t = o_t$
        }
    }
    \Return{$(v_1, \ldots, v_T)$}
    \caption{Composition with Privacy Filter}
    \label{alg:comp_privacy_filter}
\end{algorithm}

\newpage
\section{PURE DP ONLINE SELECTION: ADDITIONAL EXPERIMENTS}

\paragraph{UCI Bikes Dataset} We use the same UCI dataset as before, however this time configure Above Threshold to HALT when bike sharing exceeds a certain threshold between zero and one. In \cref{fig:fsrc_uci_scatter} we plot a range of calibrations, and in \cref{fig:fsrc_uci_spend} we show the privacy spend over time.

\begin{figure}[H]
    \centering
    \includegraphics[width=1\linewidth]{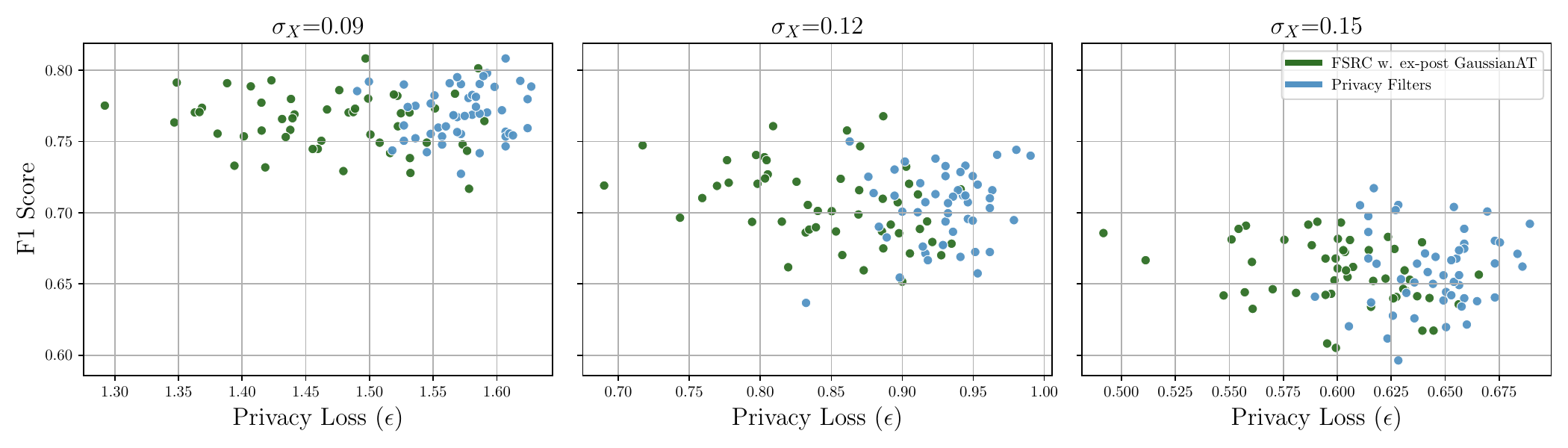}
    \caption{
    Scatter plot indicating the accuracy and final privacy spend over a range of noise multipliers for thresholds with the UCI Bikes dataset. Final privacy loss ($\epsilon$) is reported for Filtered Self-Reporting Composition (green). Threshold noise, $\sigma_X$, evaluated in the range of $[0.09, 0.15]$. We note a clear separation in privacy accounting over a range of noise multipliers.}
    \label{fig:fsrc_uci_scatter}
\end{figure}

\vspace{-0.5\baselineskip}
\begin{figure}[H]
    \centering
    \includegraphics[width=1\linewidth]{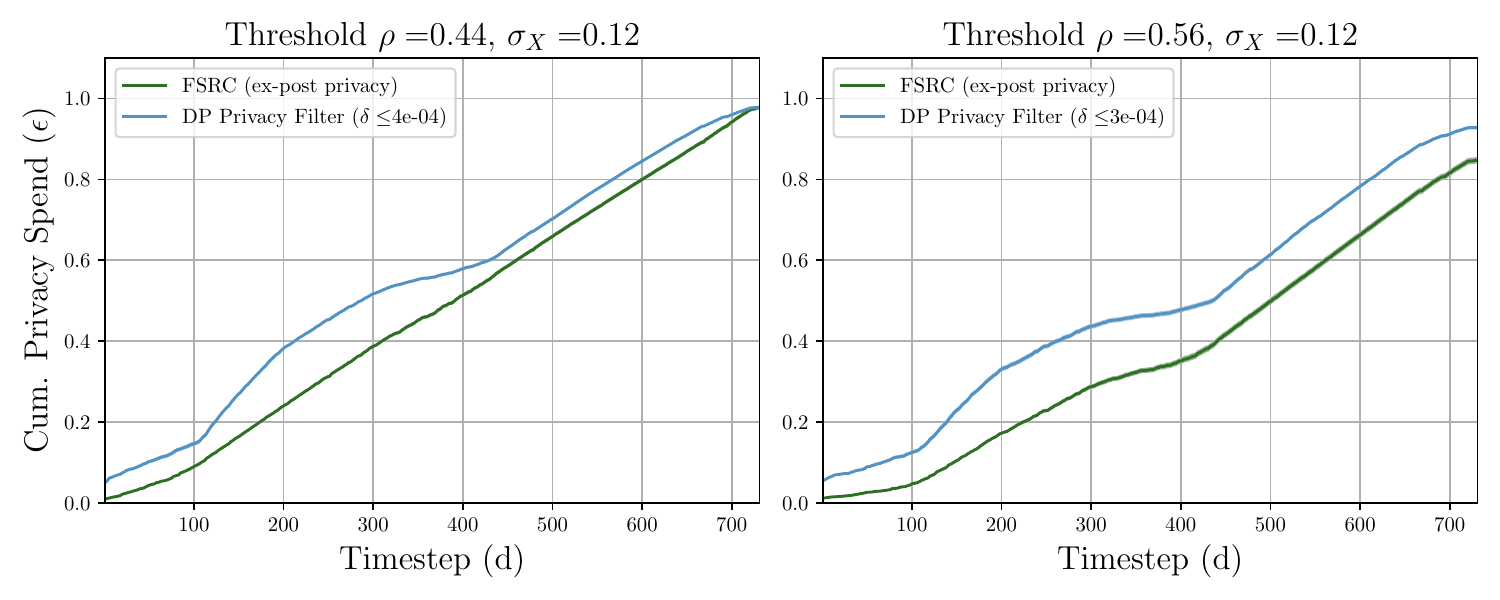}
    \caption{Privacy spend comparison between Gaussian Above Threshold and Filtered Self-Reporting Composition and Privacy Filters with Gaussian Above Threshold on the UCI Bikes dataset.}
    \label{fig:fsrc_uci_spend}
\end{figure}

\newpage
\paragraph{London Energy Dataset.} We plot Gaussian Above Threshold under the same utility metric on the LCL London energy dataset, with 5,564 customers. In \cref{fig:fsrc_lcl_scatter_sup} we plot a range of calibrations. As the threshold decreases, we witness more queries released and therefore a greater privacy spend. we show the effect over time in \cref{fig:fsrc_lcl_spend_sup}.

\begin{figure}[H]
    \centering
    \includegraphics[width=1\linewidth]{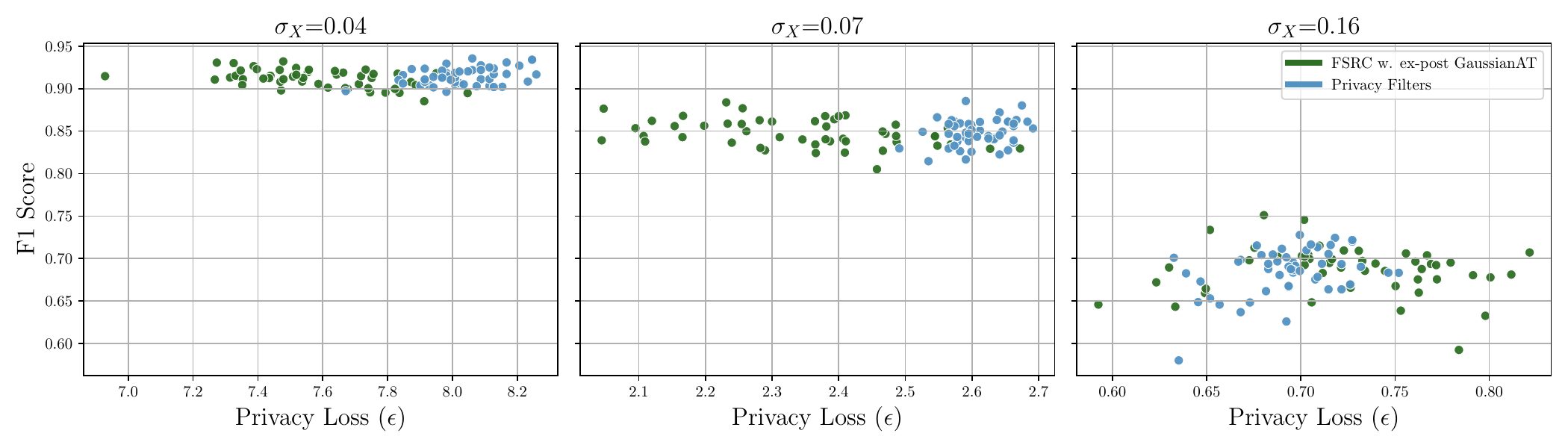}
    \caption{Scatter plot indicating the accuracy and final privacy spend over a range of noise multipliers for thresholds with the LCL London dataset. Final privacy loss ($\epsilon$) is reported for FSRC (green). Threshold noise, $\sigma_X$, was evaluated in the range of $[0.04, 0.16]$. With this dataset, our accounting method provides benefits when $\sigma_X > 0.06$.}
    \label{fig:fsrc_lcl_scatter_sup}
\end{figure}

\begin{figure}[H]
    \centering
    \includegraphics[width=1\linewidth]{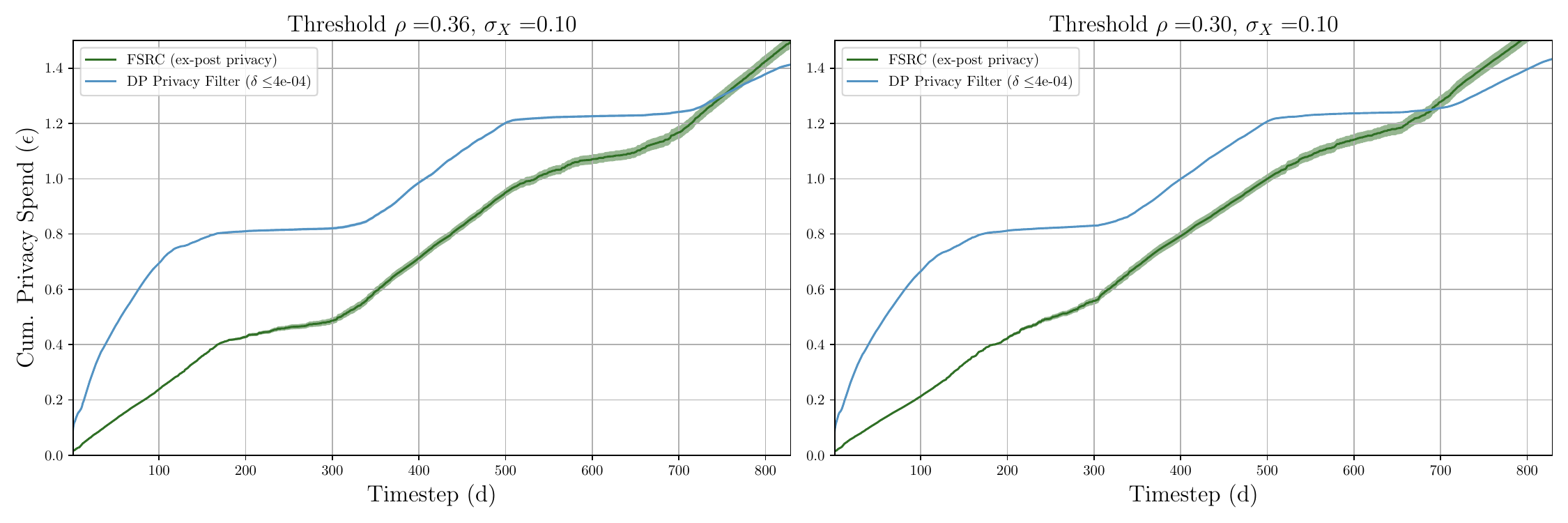}
    \caption{Privacy spend comparison between Gaussian Above Threshold and FSRC and Privacy Filters with Gaussian Above Threshold on the LCL London Energy dataset.}
    \label{fig:fsrc_lcl_spend_sup}
\end{figure}


\end{document}